%% file: RGO.tex
\documentclass{article} % For LaTeX2e
\usepackage{iclr2022_conference,times}

\input{math_commands.tex}

\usepackage{hyperref}
\usepackage{url}

\usepackage{subfig}
\usepackage{algorithm}
\usepackage{algorithmic}
\usepackage{overpic}
\usepackage{amsmath}
\usepackage{amsthm}
\newtheorem{theorem}{Theorem}
\newtheorem{lemma}[theorem]{Lemma}

\theoremstyle{definition}
\newtheorem{definition}{Definition}[section]
\theoremstyle{remark}

\usepackage{color}
\definecolor{Color1}{RGB}{12,118,30}
\definecolor{Color2}{RGB}{220,97,89}
\definecolor{Color3}{RGB}{0,158,255}

\title{Continual Learning with Recursive Gradient Optimization}

\iclrfinalcopy
\author{Hao Liu \\
Department of Computer Science\\
Tsinghua University\\
Beijing, China \\
\texttt{hao-liu20@mails.tsinghua.edu.cn} \\
\And 
Huaping Liu \\
Department of Computer Science\\
Tsinghua University\\
Beijing, China \\
\texttt{hpliu@tsinghua.edu.cn} \\
}

\begin{document}

\maketitle

\input{./text/abstract.tex}
\input{./text/introduction.tex}

\input{./text/method.tex}
\input{./text/implemention.tex}
\input{./text/experiment.tex}

\input{./text/discussion.tex}
\bibliography{ContinualLearning}
\bibliographystyle{iclr2022_conference}

\appendix
\input{./text/appendix.tex}

\end{document}

%% file: math_commands.tex
\usepackage{amsmath,amsfonts,bm}

\def\eqref#1{equation~\ref{#1}}
\def\1{\bm{1}}

\DeclareMathAlphabet{\mathsfit}{\encodingdefault}{\sfdefault}{m}{sl}
\SetMathAlphabet{\mathsfit}{bold}{\encodingdefault}{\sfdefault}{bx}{n}

%% file: text/abstract.tex
\begin{abstract}

Learning multiple tasks sequentially without forgetting previous knowledge, called Continual Learning (CL), remains a long-standing challenge for neural networks. Most existing methods rely on additional network capacity or data replay. In contrast, we introduce a novel approach which we refer to as Recursive Gradient Optimization (RGO). RGO is composed of an iteratively updated optimizer that modifies the gradient to minimize forgetting without data replay and a virtual Feature Encoding Layer (FEL) that represents different network structures with only task descriptors. Experiments demonstrate that RGO has significantly better performance on popular continual classification benchmarks when compared to the baselines and achieves new state-of-the-art performance on 20-split-CIFAR100 (82.22\%) and 20-split-miniImageNet (72.63\%). With higher average accuracy than Single-Task Learning (STL), this method is flexible and reliable to provide continual learning capabilities for learning models that rely on gradient descent.

\end{abstract}

%% file: text/introduction.tex
\section{Introduction}

In many application scenarios, one needs to learn a sequence of tasks without access to historical data, called \emph{continual learning}. Although variants of stochastic gradient descent  (SGD) have made a significant contribution to the progress made by neural networks in many fields, these optimizers require the mini-batches of data to satisfy the independent identically distributed  (i.i.d.) assumption. In continual learning, the violation of this requirement leads to significant degradation of performance on previous tasks, called catastrophic forgetting. Recent works attempt to tackle this issue by modifying the training process from a variety of perspectives.

%In many application scenarios, one needs to learn a sequence of tasks without access to historical data, called \emph{Continual Learning}. Although variants of Stochastic Gradient Descent (SGD) have made a significant contribution to the progress made by neural networks in many fields, these optimizers require the mini-batches of data to satisfy the independent identically distributed  (i.i.d.) assumption.

\textbf{Memory-based} approaches use extra \emph{memory} to store some samples \citep{Lopez-Paz2017,Chaudhry2020}, gradients \citep{Chaudhry2019,Chaudhry2020,Saha2021}, or their generative models  \citep{Shin2017,Shen2020} to modify future training process. The memory for replay leads to a linear-increased space complexity with respect to the number of tasks. \textbf{Expansion-based} approaches select the network parameters dynamically  \citep{Yoon2018,Rosenbaum2018,Serra2018,Kaushik2021}, add additional components as new tasks arrive  \citep{Rusu2016,Fernando2017,Alet2018,Chang2019,Li2019}, or use larger networks to generate network parameters  \citep{Aljundi2017,Yoon2019,VonOswald2019}. These methods reduce interference between tasks by additional task-specific parameters. Single-Task Learning (STL) can also be regarded as an expansion-based method which trains a network for each task separately.
%TODO
%Such methods will not be discussed in this paper as the total network \emph{capacity} is larger than the origin. 
\textbf{Regularization-based} approaches encourage important parameters to lie in a close vicinity of previous solutions by introducing quadratic penalty term to the loss function  \citep{Kirkpatrick2017,Zenke2017,Yin2020} or constraining the direction of parameter update  \citep{Farajtabar2019,Chaudhry2019,Saha2021}. Our method is also regularization-based which combines the advantages of loss penalty and gradient constraint.

%TOMOVE
%Most of these methods use task identifiers to select correct output head and reduce the interference among tasks, which leaves behind a portion of network capacity growing linearly with the number of tasks.
%TODO

In this work, we focus on continual learning in a fixed-capacity network without data replay. We aim to minimize the expected increment of the total loss of past tasks without reducing the performance of the current task. To this end, we propose an upper bound of quadratic loss estimation and design a recursive optimization procedure to modify the direction of the gradient to the optimal solution under this upper bound. In addition, we introduce trace normalization process to guarantee the learning rate during the training process according to the principle of \emph{current-task-first (CFT)}. This normalization process makes our approach compatible with the vast majority of existing models and learning strategies well-designed for single-task solutions. As the gradient modification process is independent of data samples and previous parameters, our optimizer can be used directly in most deep architecture networks as typical single-task optimizers like SGD. 
%TOMOVE: to achieve efficient continuous learning performance without growth of training-time space consumption and final network capacity,
Further, to reduce the interference between tasks, we develop a feature encoding strategy to represent the multi-modal structure of the network without additional parameters. A virtual \emph{feature encoding layer (FEL)} which randomly permutes the output feature maps using integer task descriptor as seed is attached after each real layer. Thus, each task obtains a specific virtual structure under same network parameters. Since the parameter space of the network is not changed, such a strategy will not change the fitting ability of the neural network. Experimental validations on several continual learning benchmarks show that the proposed method has significantly less forgetting and higher accuracy than existing fixed-capacity baselines. We achieve state-of-the-art performance on 20-split-CIFAR100 (82.22\%) and 20-split-miniImageNet (72.62\%). In addition to minimizing forgetting, this method has comparable or better performance than single-task learning which handles all tasks individually.

%TOMOVE
%to achieve efficient continual learning performance without significant growth of training-time space consumption
%TOMOVE
%the close-vicinity assumption of the optimization problem is better guaranteed.
%Alleviating the influence incurred by the conflicting tasks without additional parameters, such a strategy does not change the fitting capability of the neural networks.

%% file: text/method.tex
\section{Preliminaries}
\label{sec:Preliminaries}
% Continual learning aims to handle sequentially arrived data-label pairs $(x_i,y_i)\in \mathcal{X}\times \mathcal{Y}$ with task identifiers $t_i\in\mathcal{T}$ and estimate a predictor $f(x,t):\mathcal{X}\times\mathcal{T}\to \mathcal{Y}$

%where the task are assumed in order($t_i \leq t_j,\forall i\leq j$)
%\subsection{continual learning setup}
Consider $K$ sequentially arrived supervised learning tasks $\{\mathcal{T}_k|k\in [K]\}$, where $[N]:= \{1,2,\cdots,N\}$ for any positive integer $N$. In each task $\mathcal{T}_k$, there are $n_k$ data points $\{(x_{k,i},y_{k,i})|i\in[n_k]\}$ sampled from an unknown distribution $\mathcal{D}_k$. Let ${\mathcal{X,Y,W}}$ be the space of inputs, targets and model parameters. By denoting the predictor as $f(x,k):\mathcal{X}\times[K]\to \mathcal{Y}$, the loss function of $\theta \in \mathcal{W}$ associated with data point $(x,y)$ and task identifier $k$ can be expressed as $l(f(\theta;x,k),y):\mathcal{W}\to \mathbb{R}$, and the empirical loss function of task $\mathcal{T}_k$ is defined as: 

% \[
%     L_k(w):=\mathbb{E}_{(x,y)\sim \mathcal{D}_k}l_k(w;x,y) \]

\begin{equation}
    L_k(\theta) = \frac{1}{n_k}\sum_{i=1}^{n_k}l(f(\theta;k,x_{k,i}),y_{k,i}) 
\end{equation}

In the continual learning scenario studied in this work, the parameter space $\mathcal{W}$ remains a fixed size, and an integer task descriptor is provided at both training and testing time. % to identify the correct classification head. 
Without access to past samples, we use a second-order Taylor expansion to estimate the loss function of previous tasks.% following \cite{Yin2020}. 
Let $\theta_j^*$ be the optimal parameter of $L_j(\theta)$ generated by the gradient descent process according to $\nabla_{\theta_j^*}L_j=0$. For a new model parameter ${\theta}$ in the neighborhood of ${\theta_j^*}$, the loss of previous task $\mathcal{T}_j(j < k)$ can be estimated as:
%TODO : 加reference throughout the continual learning experience
\begin{equation}
    L_j(\theta)= L_j(\theta_j^*)+ \frac{1}{2}(\theta-\theta_{j}^*)^T H_j  (\theta-\theta_{j}^*)
\end{equation}
% 二阶泰勒近似的提出，这个很自然的假设需要邻域支撑，有一些文章已经有了成熟的讨论（模仿SOLA和OGD的写法）
where $H_j:=\nabla^2L_j(\theta_j^*)$ is the Hessian matrix.

% tasks $\{T_1,T_2,\cdots\}$ without access to early learned tasks. The labeled training dataset for task $T_k$ is represented as $\mathcal{D}_k = \{(x,y)\}$, where $x$ and $y$ are the corresponding input and output vectors, respectively.

% where $l$ is differentiable loss function for specific example $(x,y)$.

% For the $C$-class classification problems, $f(x;\theta)$ has $C$-logits associated to different classes. We consider the most commonly used softmax cross entropy loss which is defined as
% \begin{equation}
%     l(y,f(x;\Theta)) = -\sum_{j=1}^c y_j log(a_j)
% \end{equation}
% where $a_j= exp(f_j(x;\Theta))/\sum_{c=1}^C exp(f_c(x;\Theta))$ as the $j$-th softmax output. The $(i,j)$-th element of the second derivative matrix of the loss function with respect to $f(x;\Theta)$ is then calculated as
% \begin{equation}
%     l''(y;f(x,\Theta))_{i,j} = a_j \phi_{i,j} - a_ia_j
% \end{equation}
% where $\phi_{i,j}$ is Dirac function equals to $1$ while $i=j$ else $0$.

\section{Problem formulation \& solution}
\label{sec:optimizer}
% 放在experiment 里面
% metric 
% capacity with memory cost 
% 指出capacity不增加
%In our continual learning setup, all of the model parameters is shared, which means no additive parameters is used for the following tasks. The capacity of the network is same throughout the continual learning experience.
%Since that, finding appropriate solution is the core goal of our approach. In order to find out the optimum parameters work well across the task sequence, we propose a novel continual learning optimizer to minimize an upper bound of previous losses.
In our fixed-capacity continual learning setting, finding an appropriate joint solution that works well across the task sequence is the core goal. To this end, we introduce a novel continual learning optimization problem and corresponding iterative optimization strategy.
% In our fixed-capacity continual learning setting, all model parameters are shared, which means that no additional parameters are available for new tasks. Therefore, finding appropriate joint solution that works well across the task sequence is the core goal of our approach. To this end, we introduce a novel continual learning optimization problem and corresponding iterative optimization strategy.
%to minimize an upper bound of previous losses.

\subsection{Optimization problem}
    In this paper, we formalize forgetting as the the increment of old task losses. As mentioned in Section \ref{sec:Preliminaries}, the total loss of tasks before $\mathcal{T}_k$, denoted as $F_k$, can be estimated by:

    \begin{equation}
        \label{F}
        F_k(\theta) = \sum_{j=1}^{k-1}L_j(\theta) \approx \sum_{j=1}^{k-1}[L_j(\theta_j^*)+ \frac{1}{2}(\theta-\theta_{j}^*)^T H_j  (\theta-\theta_{j}^*)]
    \end{equation}

    As Equation (\ref{F}) need the explicit value of previous model parameters, $F_k$ is too expensive to be an optimization target in continual learning. We turn to a more concise form which we refer to as \emph{recursive least loss (RLL)}:

    \begin{equation}
        F_k^{RLL}(\theta) := \frac{1}{2}(\theta-\theta_{k-1}^*)^T (\sum_{j=1}^{k-1}H_j)(\theta-\theta_{k-1}^*)
    \end{equation}

    In Appendix \ref{apdx:RLL} , we prove that $F_k^{RLL}$ and $F_k$ are equivalent for optimization if all previous tasks are fully-trained. 
    % DOING 
    Based on the conclusions above, the optimization problem during task $\mathcal{T}_k$ is formalized as:

    \begin{equation}
        \label{equ:originoptim}
        \theta_k^*: \quad
        \mathop {\min }\limits_\theta F_k^{RLL}(\theta)
        ,\quad \text{subject to } \nabla L_{k}({\theta}) = 0
    \end{equation}

    $F_k^{RLL}$ has the same form as the regularization term in many regularization-based methods. The optimization goal of these methods are variants of $L_{k}({\theta}) + \lambda F_k^{RLL}(\theta)$, derived from Bayesian posterior approximation with Gaussian prior \citep{Kirkpatrick2017,Nguyen2018} or approximation of the KL-divergence used in natural gradient descent \citep{amari1998natural,Ritter2018,Tseran2018b}. The Bayesian methods try to estimate and minimize the overall loss function, while our method prioritizes the performance of the current task by $\nabla L_{k}({\theta}) = 0$ and minimizes the expected forgetting of the past tasks $F_k^{RLL}(\theta)$.

    % old version: deprecated
    %  \begin{assumption}[Joint Optimal]
    %     If all previous task is fully trained as Equation [\ref{totalupdate}]($\nabla L_{k}(\theta_k^*)=0$), the optimal parameter at the end of task $\mathcal{T}_k$ is the joint optimal parameter of all experienced tasks $\mathcal{T}_j,j\leq k$, which means $\sum_{j=1}^k \nabla L_{j}(\theta_k^*) = 0 $
    %  \end{assumption}

% 对公用的参数的持续学习
\subsection{Gradient modification}

%We begin with the modification process of gradient generated in a single step during training.

% 旧任务总loss的提出，

% 提出我们的优化目标，为什么要在保证现有任务训练速度的基础上去优化旧任务。

% 优化器模式的提出，并给出定理说明为什么这种优化器能够在期望的角度保证收敛速度，在附录中证明

% 给出在这种优化器模式下，旧任务loss的上界，上界在附录中证明。

% 这样就可以形式化提出优化问题， 及其解，解的求解过程在附录中说明。

% 不可回避的遗忘，比如两个完全相反的任务，提一下过去的工作中如何处理的，加一个不同的head

% 随机特征映射层的提出，详细定义一下以及提及这一过程不需要额外的参数，
% 局域化处理，及随机转换之后随机化的证明（放在附录，正文不提出）

%% 实现部分
% RLS算法的提出和局域化处理，这样就可以让矩阵求逆复杂度降低

% 对每一个feature 共享 P矩阵，这样的理由也很显然？

For the newest task $\mathcal{T}_{k}$, the optimal solution where $\nabla L_{k}(\theta) = 0$ should be obtained by stochastic gradient descent started from the former optimal model parameter ${\theta_{k-1}^*}$ at the end of task $\mathcal{T}_{k-1}$. Using subscript $\cdot_{i}$ to represent the parameters of step $i$, and the initial state $\theta_0=\theta_{k-1}^*$, the single step update can be expressed as:
\[
\theta_{i} = \theta_{i-1} -\eta_{i} \nabla L_{k}(\theta_{i-1})
\]

     Assume that the pre-set learning rate $\eta_i$ is small enough to ignore the higher order terms, the loss function after the one-step update can be expressed as:

     \[
        L_{k}(\theta_i) = L_k(\theta_{i-1}) -\eta_{i}(\nabla L_{k}(\theta_{i-1}))^T \nabla L_{k}(\theta_{i-1})
        \]

    %TODO keyishan
     If we hope to solve the task $\mathcal{T}_{k}$ only, updating the parameter $\theta$ according to the gradient above is enough. However, as mentioned above, such a method will encourage the neural network to gradually forget the old tasks. Therefore, we modify the update direction to minimize the expectation of forgetting. To this end, we introduce a new positive definite symmetric matrix $P$ with appropriate dimensions to modify the gradients ($g \to Pg$). The modified one-step update is:
     \begin{equation}
         \left\{
         \begin{aligned}
     &\theta_{i} = \theta_{i-1} -\eta_{i} P \nabla L_{k}(\theta_{i-1}) \\
     &L_{k}(\theta_i) = L_k(\theta_{i-1}) -\eta_{i}(\nabla L_{k}(\theta_{i-1}))^T P\nabla L_{k}(\theta_{i-1})
         \end{aligned}
         \right.
        \end{equation}

    %TODO ,讨论收敛性
    %The positive definiteness of the matrix $P$ guarantees the convergence of the algorithm. 
    To maintain the pre-set learning rate during the continual learning problem and avoid repetitive selection of hyper-parameters, we impose an additional constraint on the trace of the projection matrix and prove the corresponding convergence rate consistency theorem \ref{learningrate} in Appendix \ref{apdx:learningrate}. 

    \begin{theorem}[convergence rate consistency]
        \label{learningrate}
        Under the constraint of \emph{trace($P$)=dim($P$)}, the expectation of the learning rate for unknown isotropic distribution is the same as the original optimizer.
    \end{theorem}
    
    %TODO heshideyinyong
    For common network structures, we provide the space complexity of matrix P and the time complexity of the corresponding gradient modification process in Appendix \ref{apdx:complexity}. As described above, the only extra memory of our approach is the projection matrix $P$, which contains the information of previous tasks. This allows our approach to be a space-invariant method different with typical \emph{memory-based} or \emph{expansion-based} continual learning method. As the performance of our method is identified with the choice of $P$, the following problem is: \emph{How to find a good projection matrix?} We will answer this question in the following parts.

\subsection{Approximate solution}

    %  Then, the problem is formulated as at each mini-step, determine the matrix $P$ to minimize the forgetting of old tasks. The ideal optimization goal can be written as:

    The next step is to find a solution for problem (\ref{equ:originoptim}). When the training process on $\mathcal{T}_k$ is finished, the final state ${\theta_k^*}$ and the residual loss can be obtained from the accumulation of one-step update:

    \begin{equation}
        \label{equ:totalupdate}
         \left\{
         \begin{aligned}
     &\theta_{k}^* = \theta_{k-1}^* - \sum_{i=1}^{n_k} \eta_{i} P \nabla L_{k}(\theta_{i-1}) \\
     &L_{k}(\theta_k^*) = L_k(\theta_{k-1}^*) -\sum_{i=1}^{n_k} \eta_{i}(\nabla L_{k}(\theta_{i-1}))^T P \nabla L_{k}(\theta_{i-1})
         \end{aligned}
         \right.
    \end{equation}

    In this way, for a given sample sequence and initial value $\theta_{k-1}^*$, the result of $\theta_k^*$ depends on $P$. The optimization problem \ref{equ:originoptim} on $\theta_k^*$ is transformed into an optimization problem on $P$.

    %todo move to appendix
    % Considering that $L_{k}(\theta_{k}^*)\geq 0$ at the end of training and loss for current task before training can be expressed as $L_{k}(\theta_{k-1}^*)$, the change of loss function of task $\mathcal{T}_{k}$ satisfies the following inequality:

    % \begin{equation}
    %     L_{k}(\theta_{k-1}^*) \geq \sum_{i=1}^{n_k} \eta_{i}(\nabla L_{k}(\theta_{i-1}))^T P\nabla L_{k}(\theta_{i-1})
    % \end{equation}
    % move to appendix

    %These two constraints cannot be used in practical optimization because it is time consuming to find the solution area that meets the conditions. So we turn to minimize an upper bound independent of $\theta_k^*$. We prove theorem \ref{upperbound} in Appendix \ref{apdx:upperbound}.
    However, the relationship between $F_k^{RLL}(\theta_k^*)$ and $P$ is too complicated to be used in optimization process. To tackle this problem, we propose an upper bound of $F_k^{RLL}(\theta_k^*)$ as a practical optimization target. 

    \begin{theorem}[upper bound]
        Denote $\hat{\sigma}_{m}(\cdot)$ as the symbol for maximum eigenvalue and $\eta_m$ as the maximum single-step learning rate, the recursive least loss has an upper bound:
        \label{upperbound}
        \begin{equation}
            F_k^{RLL}(\theta_k^*)\leq \frac{1}{2}n_k\eta_m\hat{\sigma}_{m}(P\bar{H})L_k(\theta_{k-1}^*)
        \end{equation}
        where $ \bar{H} = \sum_{j=1}^{k-1}H_j $ is defined as the sum of the Hessian matrices of all old tasks.
    \end{theorem}

    We prove theorem \ref{upperbound} in Appendix \ref{apdx:upperbound}. Discarding the constant terms, we get an alternative optimization problem for projection matrix ${P}$:

     %-------------------------------

    \begin{equation}
        \label{problem:optimization}
        P: \left\{
        \begin{aligned}
        &\mathop {\min }\limits_P \quad \hat{\sigma}_{m}(P\bar{H})  \\
        &\text{subject to}\quad \text{trace}(P)=\text{dim}(P)
        \end{aligned}
        \right.
    \end{equation}

    % \sum_{i=1}^{n_k}\nabla L_k^T P\nabla L_k=\sum_{i=1}^{n_k} \nabla L_k^T\nabla L_k

    % \paragraph{Thm.} Assume that the gradients $\nabla_\theta L$ satisfies isotropic distribution $\mathcal{N}(0,\sigma)$, the mathematical expectation of the reduction of loss :
    % \[
    %     E_{\nabla_\theta L}[ \eta \nabla L^T  P \nabla L ] = \eta \frac{trace(P)}{dim(P)} E_{\nabla_\theta L}[ \eta \nabla L^T \nabla L ]
    % \]
    % See supplementary materials for detailed proof.

    % As we don't have the information on untrained tasks, assumption of isotropic distribution is applied to get the alternative constraint 
    % Thus, the Optimization problem need an extra normalization consistant(see Appendix for more details). Minimize the upper bound derivated:
    We provide a detailed solution of this problem in Appendix \ref{apdx:solution}. The normalized solution is:

\begin{equation}
    P=\frac{dim(\bar{H})}{trace(\bar{H}^{-1})}\bar{H}^{-1} 
\end{equation}

% TODO 
% The normalized projection procedure can be described by Figure \ref{fig:diagram}: finding a new grad to minimize the incremental loss of the old tasks with consistent impact on current task.

% % TODO : t+1 still on the figure
% \begin{figure}[htbp]
%    %\vskip 0.2in
%    \begin{center}
%    % \centerline{\includegraphics[width=0.8\columnwidth]{aifig2.eps}}

%    \centering\footnotesize
% \begin{overpic}[scale=0.8]{aifig2-1.eps}

% \put(20,10){{ \textbf{\textcolor{Color1}{${F_k^{RLL}}$}}} }
% \put(38,32){${\theta_{k-1}^*} $}
% \put(82,26){${\theta_{k}^*} $}
% \put(49,35){\textcolor{Color3}{${g} $}}
% \put(45,25){${g'} $}
% \put(56,65){\textcolor{Color3}{${\Sigma} $}}
% \put(85,55){{ \textbf{\textcolor{Color2}{${L_{k}}$}}} }
% %\put(5,20){${\mathcal{L}(\bar{\theta})}$}
% \put(100,100){${11} $}

% \end{overpic}
%    \caption{Diagram of Local Optimizer: Denote $g$ as the gradient generated by origin optimizer, $\Sigma$ is the plate containing any possible $g$ with "similar" impact on task k, and $g'$ is the one selected which minimizes $F_k^{RLL}$.}
%    \label{fig:diagram}
%    \end{center}
%    %\vskip -0.2in
% \end{figure}

% As shown on the diagram, our optimizer only modifies the direction of gradient and does not reduce the search region, which will keep the fitting ability of the network consistent across the task sequence. Without adding regularization representing previous tasks into the loss function, the minimum loss for the current task, in theory, is lower than regularization based methods.

The normalized projection procedure can be described as: finding a new gradient that has similar effects on the current task to minimize the upper bound of the old task losses. Our optimizer only modifies the direction of the gradient without reducing the search region, which will guarantee the fitting ability of the network consistent throughout the task sequence.

%Without adding regularization representing previous tasks into the loss function, the minimum loss for the current task, in theory, is lower than regularization based methods.
%TODO zhege in theory 不太对

%% file: text/implemention.tex
\section{Implementation}

The main concern of our method in section \ref{sec:optimizer} lies in the expensive space and time cost in deep neural networks. In this section, we propose two approaches to reduce the time and space complexity of the algorithm for models with forward-backward propagation structures. 

\subsection{Virtual Feature Encoding Layer}

% 说明不能直接独立处理
% A natural approach is to assume parameters in different layers independent and process them as separate problems, which is used in some \emph{regularization-based} methods based on second-order Taylor expansion like \cite{}.

In multi-layer networks, the output of the previous layer can be regarded as a set of features generated by the feature extractors(for example, weight matrices, bias vectors, convolution kernels, etc). 
%Expanded according to the last dimension, the feature set are randomly encoded with task descriptors as seeds. 
%TODO 介绍打乱的好处
In order to make the gradients generated in back-propagation process conform to our assumption of isotropic distribution(Theorem \ref{learningrate}), we propose a virtual \textbf{Feature Encoding Layer(FEL)} to apply task-specific connections to the output of the previous layer and the input of the next layer. 

\begin{definition}[Feature Encoding Layer]
    A feature encoding layer applies a task-specific rearrangement to the input feature maps, the order of which is randomly generated using the task identifier as a seed.
\end{definition}
% TOMOVE
% In order to establish a general method to distinguish different tasks with the same network structure,we propose an intuitive method to encode the sub structures of the network according to different tasks. Under our continual learning setup we considered, task identifiers is given at both train and test time. 

% We first introduce a \textbf{Task-Encoded Feature Encoding Layer} to represent the different mapping relations between the output of the previous layer and the input of the next layer as shown in Figure \ref{fig:FRML}. 

%The structure diagram is shown in Figure \ref{fig:FEL}. 
%Although this process is only a remapping of the forward\&backward propagation and does not require extra space, we write the feature remapping layers in matrix form for the convenience of theoretical analysis:
%is only a remapping of the forward\&backward propagation and
Note that FEL is only a permutation of existing feature maps, and its order does not change during the training process. Although this feature encoding layer does not require extra space, we write it in matrix form for the convenience of theoretical analysis. The permutation matrix at layer $l(l=1,2,\cdots,L)$ is:

\[
    S_{l}(k):= \text{random permutation of}\ I_{l}\ \text{with seed}(k)\]

where $I_l$ is an identity matrix with dimensions equal to the number of feature maps. Considering that the order of the features is critical for the next layer to obtain interpretable information, the recognition ability of the network will degrade greatly without correct feature order. FEL provides an effective way to eliminate the interference between tasks, where features encoded by a specific task descriptor will be randomly permuted in other tasks. Therefore, the same feature extractor plays different roles in different tasks. Although gradients of different layer are strongly correlated in the current task, there is little correlation from the perspective of old tasks. That means current impact on past tasks can be regarded as the independent summation of the influence from different local feature extractors. 

\subsection{Local-global equivalence}
\label{subsec:localglobal}
Then, during the backward propagation process of task $\mathcal{T}_k$, the gradient of $L_k$ on an intermediate layer $h_l$ can be calculated by the chain rule:

\begin{equation}
    g_l=\frac{\partial L_k}{\partial h_l} = \frac{\partial L_k}{\partial h_L} \prod_{j=l}^{L-1}\frac{\partial h_{j+1}}{\partial h_j} = g_L \prod_{j=l}^{L-1} S_{j+1}(k)D_{j+1}W_{j+1}
\end{equation}

where $D_{j+1}$ is a diagonal matrix representing the derivative of the nonlinear activation function. 
Since previous samples cannot be accessed to recalculate the gradient, common gradient-based methods \citep{Li2019,Azizan2019a,Farajtabar2019} assume that the joint optimal parameter lies in the neighborhood of the previous optimal parameter and use previously calculated gradients as an approximation, which leads to $\nabla f(\theta;x) \approx \nabla f(\theta^*;x)$ and $\frac{\partial^2 h_L}{\partial h_l^2} \approx 0 \label{approximate:grad1}$. We follow this assumption and use the proposed optimizer to ensure this assumption as much as possible. Thus we have:

\begin{equation}
    \begin{aligned}
    \frac{\partial^2 L_k}{\partial h_l \partial h_r} &=(\frac{\partial h_L}{\partial h_l})^T \frac{\partial^2 L_k}{(\partial h_L)^2} (\frac{\partial h_L}{\partial h_r}) %= U_l^T \Lambda_{\bar{H}_l}^{\frac{1}{2}}  \Lambda_{\bar{H}_r}^{\frac{1}{2}} U_r
    %&=(\prod_{j=l}^{r-1}\frac{\partial h_{j+1}}{\partial h_j} )^T\frac{\partial^2 L_k}{ (\partial h_r)^2 } \\ 
    %&= \frac{\partial^2 L_k}{ (\partial h_l)^2 } (\prod_{j=l}^{r-1}\frac{\partial h_{j+1}}{\partial h_j} )^{-1} \\
    \end{aligned}
\end{equation}

\begin{equation}
    \label{simplified Hessian}
    \bar{H}_l=\sum_{j=1}^{k-1} \frac{\partial^2 L_j}{\partial h_l^2}=\sum_{j=1}^{k-1}\sum_{i=1}^{n_j}(\frac{\partial h_L}{\partial h_l})^T l''(f(\theta;x);y) \frac{\partial h_L}{\partial h_l} +\alpha I_l
\end{equation}

where $\alpha$ is a penalty parameter to ensure the positive definiteness of $\bar{H}_l$. We set $\alpha = 1$ in all subsequent sections according to the trace normalization proposed in Section \ref{sec:optimizer}.

Further, we can decompose the global optimization problem into independent sub-problems at each layer. We introduce Theorem \ref{localequivalence} and prove it in Appendix \ref{apdx:localequivalence}:

\begin{theorem}[Local-Global Equivalence]
    \label{localequivalence}
    Under the assumption of close vicinity, the global optimization problem is equivalent to independent local optimization problems. The local optimal projection matrix of layer $l$ is:
    \[
        P_l=\frac{\text{dim}(\bar{H_l})}{\text{trace}(\bar{H_l}^{-1})}\bar{H_l}^{-1} ,\quad \text{where} \  \bar{H_l}=\sum_{j=1}^{k-1} H_{j,l}
    \]
\end{theorem}

\subsection{Iterative update}

% TOMOVE
% Similar to \cite{Farajtabar2019}, in practical classification problems, we can only consider the logit corresponding to the ground-truth label which represents most of the information of previous tasks, which will reduce the runtime by $C\times C$ times. This leads to the following ground-truth form:

% \begin{equation}
%     \begin{aligned}
%     &\nabla^2 L_k(\theta) \approx \sum_{(x,y)\in \mathcal{D}_k} \nabla f_j(\theta;x) l''(f;y) \nabla f_j(\theta;x)
%     \end{aligned}
%     \label{simplified Hessian}
% \end{equation}
% for $j\in [1,c]$ where $y_j=1$.

    %TODO ,zeng for rls
    %if number of tasks is high and samples in each task is low, 
    Considering that the complexity of calculating inverse matrix with dimension of $n$ is O$(n^3)$, it is time-consuming to calculate the inverse Hessian matrix $\bar{H}_l^{-1}$ in practice. Instead, we update the projection matrix $P_l$ iteratively like \emph{recursive least square(RLS)} \citep{Simon2002} algorithm at training step(see Appendix \ref{apdx:quadratic} for more details). This allows RGO to be an online algorithm with linear memory complexity and single-step time complexity in the number of model parameters. We summarize the gradient modification and the iterative update of the projection matrix to Algorithm \ref{alg:algorithm}.
    % \begin{equation}
    %     \left\{
    %     \begin{aligned}
    %         &P_l'=P_l-k(\frac{\partial h_L}{\partial h_l})^T P_l\\
    %         &k=\frac{l''(f;y)P_l(\frac{\partial h_L}{\partial h_l})}{\alpha+l''(f;y)(\frac{\partial h_L}{\partial h_l})^TP_l(\frac{\partial h_L}{\partial h_l})}
    %     \end{aligned}
    %     \right.
    %     \label{update}
    % \end{equation}

    % \begin{algorithm}[h]
    %     \caption{RLS update for local projection matrix $P_l$}
    %     \textbf{Parameter:} $P_l$: local projection matrix at $l$-th Layer.
    %     \label{alg:RLS}
    %     \begin{algorithmic}[1]
    %     \STATE $g_l= [l''(f(\theta;x);y)]^{\frac{1}{2}}\frac{\partial h_L}{\partial h_l}$ \hfill $\leftarrow$ get local gradient by back-propagation 
    %     \STATE $k_l=P_l\cdot g_l/(\alpha+g_l^T P_l g_l)$
    %     \STATE $P_l \leftarrow P_l-k_l g_l^T P_l$ \hfill $\leftarrow$Update projection matrix
    %     \end{algorithmic}
    % \end{algorithm}

    %Dealing with each feature extractor as an independent problem is benificial for computational complexity and robustness. 
    Note that feature extractors in the same layer share a projection matrix calculated by the average of the gradients considering their linear correlation. In this way, handling multiple gradients in the same layer does not increase the complexity of updating projection. We list the memory size and the single step-time complexity for different kinds of feature extractors in Appendix \ref{apdx:complexity}. It is worth mentioning that, because the local optimizers of different layers are independent, after obtaining the back-propagation gradients, the gradient modification process of different layers can be processed in parallel to further reduce the time required.

        \begin{algorithm}[ht]
            \caption{Learning Algorithm of Recursive Gradient Optimization}
            \label{alg:algorithm}
            \textbf{Input}: Task sequence ${\mathcal{T}_k}, k=1,2,\cdots$\\
            \textbf{Output}: optimum parameter $\theta_k^*$
            \begin{algorithmic}[1] %[1] enables line numbers
            \STATE $P_l(0) \leftarrow I_l$. $\theta \leftarrow \theta_0^*$ is randomly initialized.
            \FOR{$k=1, 2, \cdots$}
            %\STATE \textbf{//update the parameters}
            \FOR{$(x,y)\sim \mathcal{D}_k$ }
    
            \STATE $g_l= \frac{\partial L_k}{\partial h_l},\quad l=1,2,\cdots,L$ \hfill $\leftarrow$ Stochastic/Batch Gradient of current loss
            \STATE $\hat{g}_l = P_{l} \cdot g_l \cdot \text{dim}(P_{l})/\text{trace}(P_{l})$ \hfill $\leftarrow$ modify origin gradients
            \STATE $\theta \leftarrow \theta-\eta \hat{g}$ \hfill $\leftarrow$update the model parameters
            \ENDFOR
            \FOR{$(x,y)\sim \mathcal{D}_k$ }
            \STATE $g_l= [l''(f(\theta;x);y)]^{\frac{1}{2}}\frac{\partial h_L}{\partial h_l}, \quad l=1,2,\cdots,L$ \hfill $\leftarrow$ get local gradient by back-propagation 
            \STATE $k_l=P_l\cdot g_l/(\alpha+g_l^T P_l g_l)$
            \STATE $P_l \leftarrow P_l-k_l g_l^T P_l$ \hfill $\leftarrow$Update projection matrix
            \ENDFOR
            \STATE $\theta_k^* \leftarrow \theta$ \hfill $\leftarrow$get the optimal model parameter 
    
            \ENDFOR
            \end{algorithmic}
            \end{algorithm}

%% file: text/experiment.tex
\section{Experiment setup}
%TODO： 减小篇幅，凝练表达，用capacity/memory 的表述。

\paragraph{Benchmarks:}
We evaluate the performance of our approach on four supervised continual learning benchmarks. Permuted MNIST \citep{Goodfellow2014,Kirkpatrick2017} and Rotated MNIST \citep{Chaudhry2020} are variants of MNIST dataset of handwritten digits \citep{Lecun1998} with 20 tasks applying random permutations of the input pixels and random rotations of the original images respectively. Split-CIFAR100 \citep{Zenke2017} is a random division of CIFAR100 into 20 tasks, each task has 5 different classes. Split miniImageNet, introduced by \citep{Chaudhry2020}, applies a similar division on a subset of the original ImageNet \citep{Russakovsky2015} dataset.

% where each task has a certain random permutation of the input pixels which is applied to all the images of that task. 

% Rotated MNIST is another variant of MNIST, where each task applies a fixed random image rotation (between 0 and 180 degrees) to the original dataset.

\paragraph{Baselines:}
In this work, we perform experiments on the benchmarks above with the following fixed capacity methods and an expansion-based method for comparison:
(1) SGD which uses stochastic gradient descent optimizing procedure to finetune the model,
(2) EWC \citep{Kirkpatrick2017} which is one of the pioneering regularization methods using fisher information diagonals as important weights,
% (3) OGD\cite{Farajtabar2019} which constraints the loss gradients to be orthogonal with that of previous tasks. 
(3) A-GEM \citep{Chaudhry2019} which uses loss gradients of stored previous data in an in-equality constrained optimization,
(4) LOS \citep{Chaudhry2020} which constraints gradients in a low-rank orthogonal subspace,
(5) ER-ring \citep{Chaudhry2019b} which utilizes a tiny ring memory to alleviate forgetting,
(6) GPM \citep{Saha2021} which trains new tasks in the residual gradient subspace,
(7) APD \citep{Yoon2019} which is a strong expansion-based method decomposing the parameters of different tasks with a common basis,
and (8) STL which trains a model for each single task. 
For the compared methods, we follow the original implementations to perform some necessary processing at the end of every task. The storage for memory-based methods is set to 1 sample per class per task following \citet{Chaudhry2020}.

% (6) MEGA-2: Strongest memory-based approach 
% (8) DCP : direction constrained optimization
% (9) HAT : with hard-attention to the task
%TOMOVE
%(6) APD\cite{Yoon2019} , which additive decompose parameters of different tasks, can be used as a baseline to compare the trade-off between accuracy and memory.
%(10) MTL, which uses the stochastic gradient descent with full access to previous samples and can be regarded as an upper bound on continual learning tasks.

%The regularization parameter for EWC is set to 10.

\paragraph{Metrics:} We use average accuracy(ACC) and average accuracy decline, also called backward transfer(BWT) by \cite{Lopez-Paz2017}, to evaluate the classification performance. Denote the accuracy of task $k$ at the end of task $T$ as $R_{T,k}$, ACC and BWT are defined as:
\[
    \text{ACC} = \frac{1}{T}\sum_{k=1}^T R_{T,k}, \quad \text{BWT} = \frac{1}{T-1}\sum_{k=1}^T R_{T,k}-R_{k,k}
\]

\paragraph{Architectures and training details:}
We evaluate all of the continual learning methods for the same network architectures. The model is a three-layer fully connected network with 256 hidden units in MNIST experiment and a standard ResNet18 \citep{He2016} in CIFAR and ImageNet experiments. For RGO, we add a virtual feature encoding layer attached to each layer. MNIST variants are trained 1000 steps while CIFAR and miniImageNet are trained 2000 steps. Batchsize is set at 10 for all tasks. The task identifiers are provided for both training and testing time. All results are reported across 5 runs with different seeds. See Appendix \ref{sec:experiment} for more details.

\section{Results \& Discussions}

\begin{table}[h]
    \caption{Performance of different methods on Permuted MNIST/ Rotated MNIST/ Split Cifar100/ Split mini-Imagenet. The model is trained with 20 tasks for 5 different seeds and evaluated by final average accuracy with stds. ($^*$) denotes methods with additional network capacity. RGO-2 is a version without FEL for ablation experiments}
    \label{tab:overall}
    %\vskip 0.15in
    %\scriptsize
    \centering
        \begin{tabular}{lccccc}
            & &\multicolumn{2}{c}{\bf Permuted MNIST} & \multicolumn{2}{c}{\bf Rotated MNIST}                  \\
             %\cmidrule(r){3-4}\cmidrule(r){5-6} 
             \hline
             {\bf Method} &{\bf Replay} & \bf{ACC$_{test}$(\%)} & {\bf BWT(\%)}  & \bf{ACC$_{test}$(\%)} & {\bf BWT(\%)} \\ 
            \hline
            RGO &N & \textbf{91.15}($\pm$0.20) &-2.05($\pm$0.09) & \textbf{91.25}($\pm$0.01) & -1.59($\pm$0.01)\\
            RGO-2 &N& 87.95($\pm$0.01) &-5.65($\pm$0.38) & 72.26($\pm$0.95)&-20.74($\pm$0.01) \\
            GPM  &N& 83.29($\pm$0.01) &-8.45($\pm$0.01) & 70.02($\pm$0.95)&-17.95($\pm$0.01) \\
            LOS$^*$  &N& 86.56($\pm$0.38) &-4.10($\pm$0.33) & 80.21($\pm$1.11)&-13.44($\pm$1.06) \\
            ER-Ring  &Y& 79.84($\pm$0.63) &-12.88($\pm$0.65)& 69.20($\pm$0.79)&-25.93($\pm$0.85)\\
            AGEM  &Y& 72.32($\pm$1.04)  &-19.94($\pm$1.02)& 53.26($\pm$1.00)&-41.74($\pm$0.96)\\
            EWC &N&67.79($\pm$1.60)   &-24.38($\pm$1.53)& 43.27($\pm$0.66)&-50.74($\pm$0.76) \\
            SGD &N&46.11($\pm$3.91)  &-46.06($\pm$4.00)&  44.82($\pm$0.01) &-50.18($\pm$0.01) \\
            \hline
            STL &N& 91.33($\pm$0.20) &0.0& 91.09($\pm$0.01) &0.0 \\

        \end{tabular}

        \vskip 0.1in
        %\scriptsize
        \centering
            \begin{tabular}{lccccc}
                & &\multicolumn{2}{c}{\bf Split CIFAR100} & \multicolumn{2}{c}{\bf Split ImageNet}                  \\
             %\cmidrule(r){3-4}\cmidrule(r){5-6} 
             \hline
             {\bf Method} &{\bf Replay} & \bf{ACC$_{test}$(\%)} & {\bf BWT(\%)}  & \bf{ACC$_{test}$(\%)} & {\bf BWT(\%)} \\ 
                \hline
                RGO & N & \textbf{73.18}($\pm$0.51)&-1.67 ($\pm$0.29)& \textbf{70.33}($\pm$0.87)& -1.64 ($\pm$0.41) \\
                RGO-2 & N & 62.82($\pm$0.98) &-15.91 ($\pm$0.92)& 57.40($\pm$1.90)&-22.40 ($\pm$1.89) \\
                GPM  & N & 53.41($\pm$2.87) &-27.98 ($\pm$3.14)& - & - \\
                LOS$^*$ & Y & 56.20($\pm$1.12) &-25.46 ($\pm$1.14)& 43.25($\pm$2.29)&-34.75 ($\pm$2.69) \\
                ER-Ring & Y &53.74($\pm$2.13)&-28.15 ($\pm$2.02)& 45.88($\pm$2.39)&-29.21 ($\pm$1.63) \\
                AGEM & Y& 49.56($\pm$2.64) &-32.10 ($\pm$2.73)& 34.67($\pm$0.52)& -38.06 ($\pm$0.87)\\
                EWC &N &   47.71($\pm$1.70)&-25.17 ($\pm$1.50)& 32.61($\pm$3.67)&-24.95 ($\pm$3.62)\\
                SGD &N &    37.02($\pm$1.64)&-44.34($\pm$1.55) & 37.69($\pm$1.00)&-37.23($\pm$0.72) \\
                \hline
                STL & N & 74.90($\pm$0.73)& 0.0 & 67.76($\pm$1.70) &0.0\\
            \end{tabular}
    %\vskip -0.1in
\end{table}

\begin{figure}[htbp]
    \centering
    \subfloat[Permuted MNIST]{\includegraphics[width=0.47\textwidth]{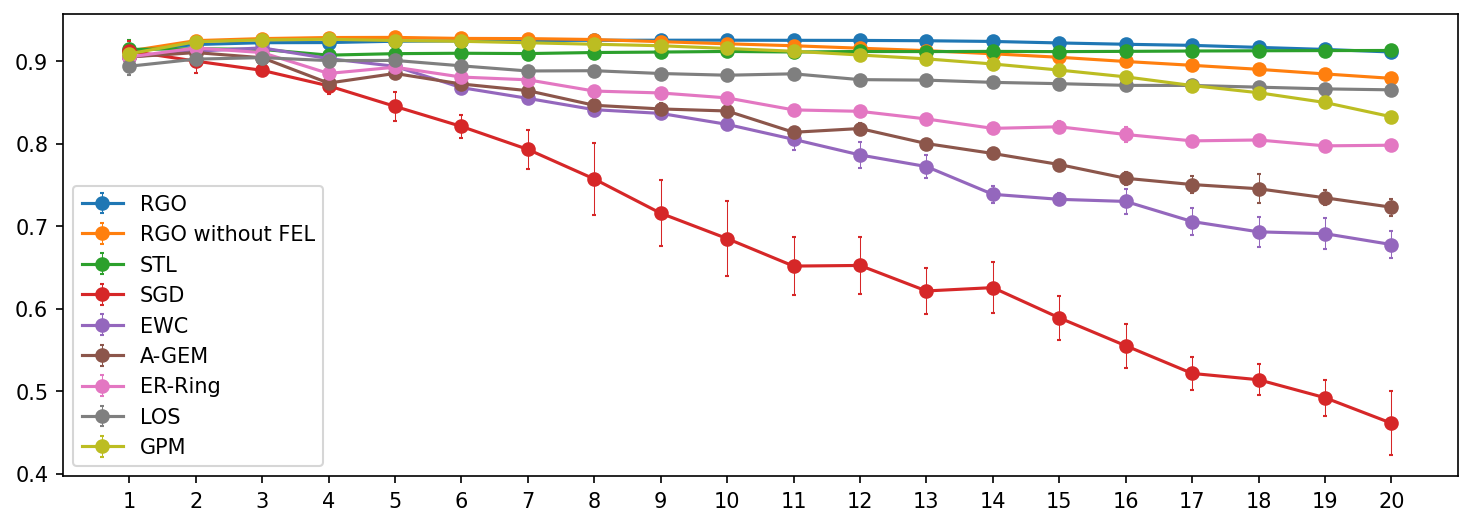}} \quad \subfloat[Rotated MNIST]{\includegraphics[width=0.47\textwidth]{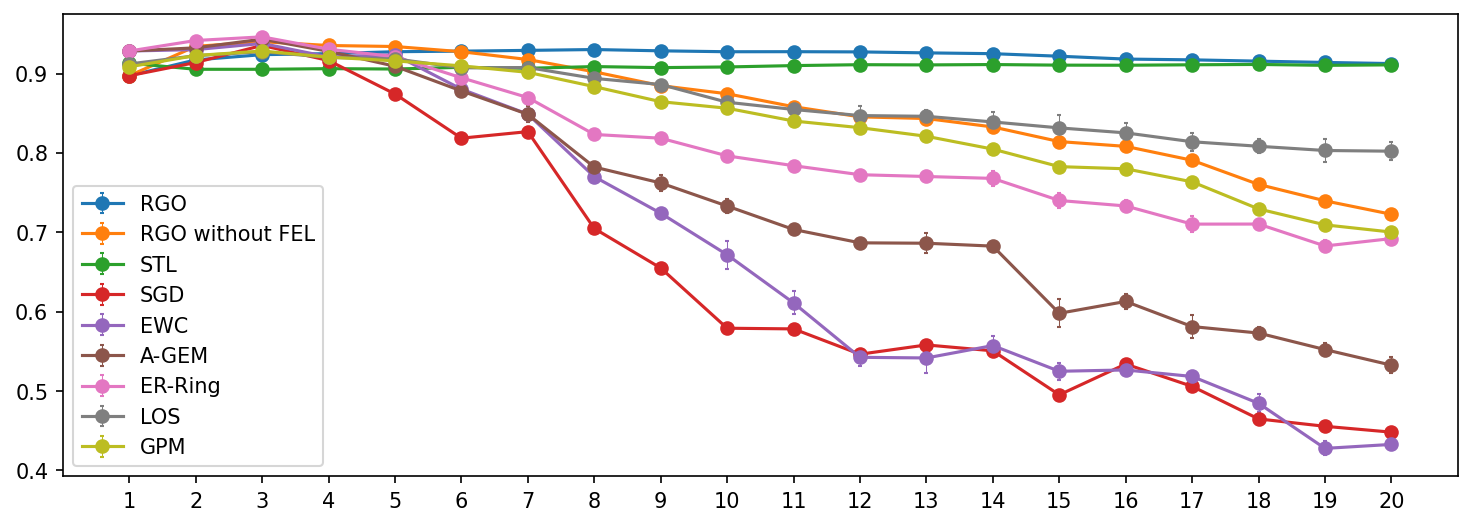}} \\
      \subfloat[Split CIFAR100]{\includegraphics[width=0.47\textwidth]{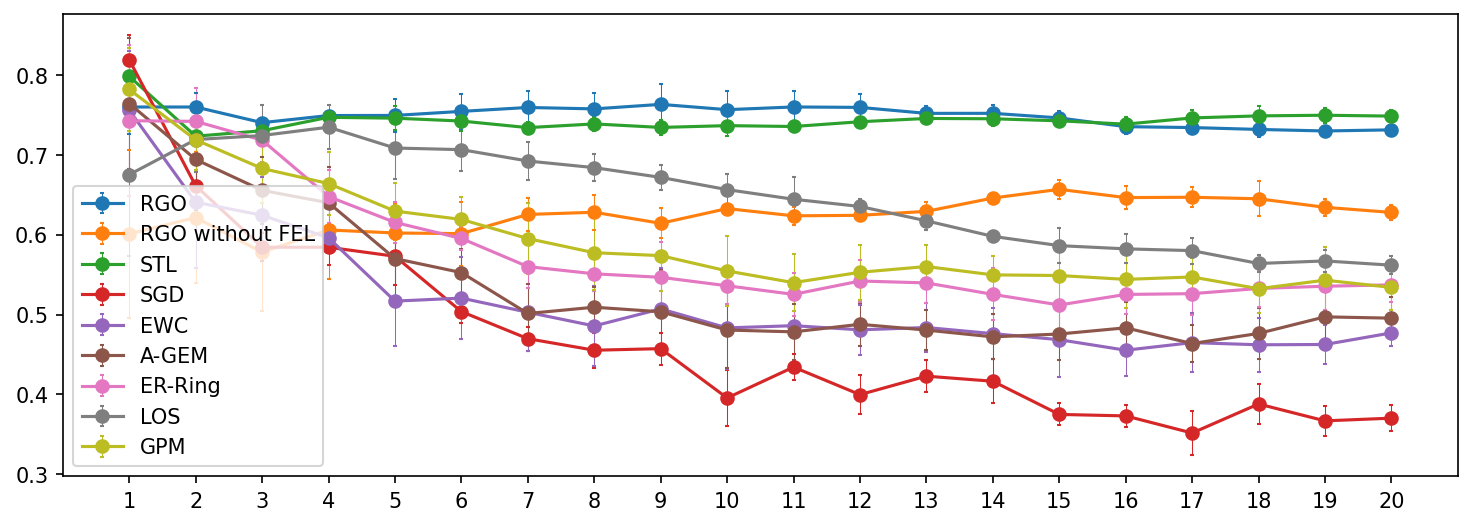}} \quad \subfloat[Split miniImageNet]{\includegraphics[width=0.47\textwidth]{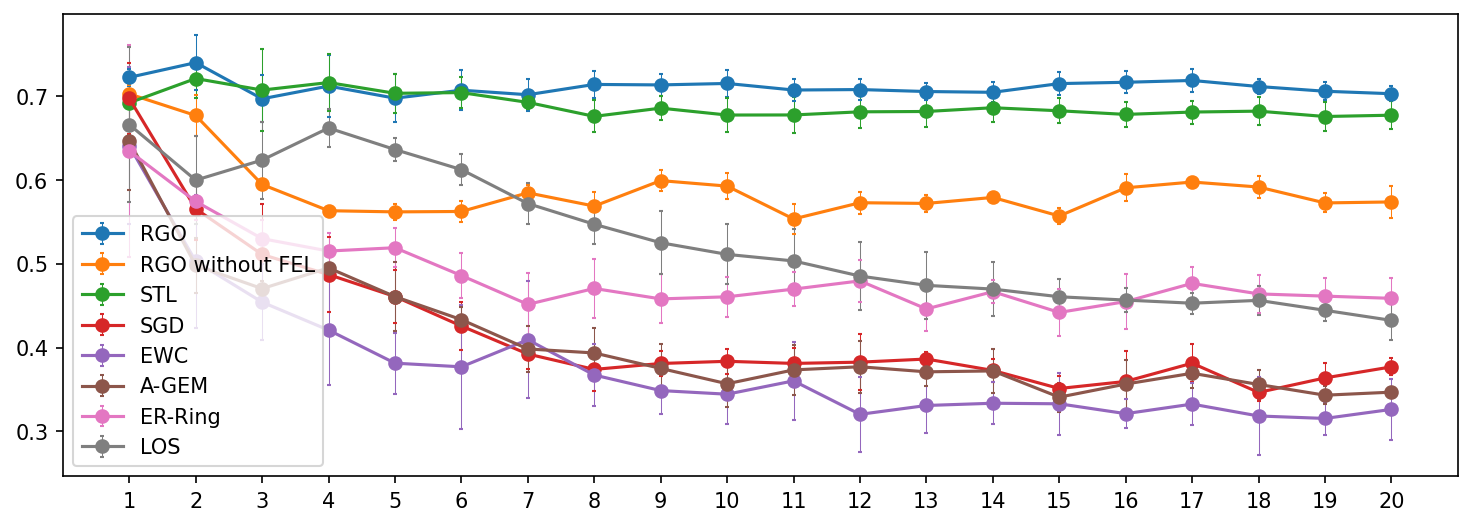}} \\
    \caption{Evolution of average accuracy with the number of tasks during the continual learning process.}
      \label{fig:geo_distribution}
      %\vspace{0.2in}
\end{figure}

%We compute the average classification accuracy on all previous tasks as a metric for continual learning algorithms, and report the average accuracy with respect to number of trained tasks. Memory\&Capacity represents extra memory required at training time by regularization-based and memory-based methods, summed with the capacity of the final network, which usually larger in model-based methods.
%提一句每一层的local optimizer可以并行处理
%The capacity of single task classifier is marked as 100\%.
% which suffers from {\color{red}the limited number of learning steps in subsequent tasks}
% In order to further test the reliability of our method, we generate 5 task sequences concluding 5 tasks with different random permutations. The average test accuracy with error bars of the indicated tasks at the end of training are listed in table \ref{tab:permuta}. 
%FMEN shows best performance on maintaining the previous tasks and highest average accuracy. The trainable space of other methods decreases with the increase of constraints, while our method keeps consistent learning rate. 
%That's why our approach looks a little less accurate in the last two tasks.

%In order to further test the reliability of our method, we generate five task sequences concluding 5 tasks with different random permutations. 
The evolution of average accuracy is shown in Figure \ref{fig:geo_distribution} and the final results with error bars of the indicated datasets at the end of training are reported in Table \ref{tab:overall}. The proposed method shows a strong performance of average accuracy over the baselines on all benchmarks. The result of BWT shows that RGO can significantly reduce catastrophic forgetting especially on complex tasks and deep architectures. RGO improves upon strongest baseline considerably: 17.0\% and 24.5\% absolute gain in average accuracy, 93.4\% and 94.4\% reduction in forgetting, on CIFAR100 and miniImageNet, respectively. Meanwhile, on rotated MNIST and miniImageNet, we observe that our method even exceeds the performance of STL which is often regarded as the upper bound of continual learning methods. The results of ablation experiments show that RGO maintains good performance without FEL, only slightly lower than LOS which has an additional task orthogonal mapping layer on Rotated MNIST. In Table \ref{tab:exp}, we list some results on modified LeNet from APD (\citep{Yoon2019}) as a comparison with expansion-based methods. Contrary to the forgetting of other methods, RGO shows positive knowledge transfer and exceeds the theoretical upper bound on the testing set under a fixed network capacity.

\begin{table}[h]
    \caption{Average accuracy of 20-task Split-CIFAR100 dataset with modified LeNet. Capacity denotes the percentage of network capacity used with respect to the original network. Relative ACC represents the difference in accuracy between the corresponding method and STL under the same settings. ($^*$) denotes results reported from APD.}
    \label{tab:exp}
    %\vskip 0.15in
    %\scriptsize
    \centering
        \begin{tabular}{lcccccc}
            %\toprule
            %  &\multicolumn{5}{c}{Method}                \\
            %  %\cmidrule(r){2-7} 
            %  \hline
             {\bf Metric}  & {\bf PGN$^*$}  & {\bf DEN$^*$} & {\bf RCL$^*$} &{\bf APD$^*$ }& {\bf RGO }\\ 
            %\midrule
            \hline
            relative ACC(\%)  & -10.24$\pm$0.39 & -9.90$\pm$0.77 & -9.01$\pm$0.25 & -4.19$\pm$0.33& \textbf{+6.02$\pm$0.5}\\
            Capacity & 271\% & 191\%& 184\%&130\%& \textbf{100\%}\\
            %bottomrule
        \end{tabular}
    %\vskip -0.1in
\end{table}

% For deep convolutional networks on CIFAR100 and miniImageNet, without any replay memory, RGO achieves results far exceeding all baselines in term of accuracy.
%   According to our analysis, this is due to the more efficient use of task identifiers by the FRL layer, rather than simply selecting an output head. This shows that the true upper bound should be the average accuracy of models trained separately on each task.

%\subsection{Experiment with more architectures}
Further, we test our method with more architectures. As shown in Table \ref{tab:archs}, RGO achieves higher test accuracy than STL with only 5\% capacity despite forgetting on the training set. RGO provides more robust features to reduce the accuracy gap between the training set and the testing set by 9\% to 44\%. Meanwhile, we report new state-of-the-art performance of \textbf{82.22\%} and \textbf{72.63\%} on Split-CIFAR100 (20 tasks) and Split-miniImageNet (20 tasks) respectively without a well-designed training schedule. Although RGO minimizes forgetting in each local optimization problem, due to the layer-by-layer accumulation of errors, deeper structures lead to more forgetting. FEL uses random permutation to greatly reduce the coupling between layers, which plays an important role in alleviating this problem. In this perspective, shallow and wide structures may be beneficial to alleviate catastrophic forgetting in the field of continual learning.

%\paragraph{Comparasion with expansion-based methods}
%Expansion-based methods without data replay usually regard STL as the upper bound. 

% \begin{table}[h]
%     \caption{Results of different architectures on 20-task split-CIFAR100 dataset. $\Delta$ denotes the difference between trainset accuracy and testset accuracy.
%     All experiments are trained 5 times with 20 epoches under learning rate 0.03.}
%     \label{tab:archs}
%     \vskip 0.15in
%     \scriptsize
%     \centering
%         \begin{tabular}{cccccccc}
%             \toprule
%              &\multicolumn{3}{c}{Single-Task Learning}  &\multicolumn{4}{c}{Recursive Least Loss}               \\
%              \cmidrule(r){2-4} \cmidrule(r){5-8}
%              Arch & ACC$_{train}$(\%) & ACC$_{test}$(\%)  & $\Delta$(\%) & ACC$_{train}$(\%) & ACC$_{test}$(\%)  & $\Delta$(\%)  & BWT(\%) \\ 
%             \midrule
%             LeNet-like & 100.0$\pm$0.0 & 75.01$\pm$0.32 & 24.99$\pm$0.32 & 99.35$\pm$0.05 & 81.03$\pm$0.51  & 18.31$\pm$0.47 & -0.99$\pm$0.15\\
%             AlexNet-like   & 99.90$\pm$0.03  & 81.60$\pm$0.31 & 18.30$\pm$0.31 & 98.78$\pm$0.08 & \textbf{82.22$\pm$0.24}& 16.55$\pm$0.18 & -1.45$\pm$0.13\\
%             VGG-11         & 98.30$\pm$0.55  & 79.51$\pm$0.71 & 18.78$\pm$0.56 & 95.17$\pm$0.08 & 79.81$\pm$0.22  & 15.35$\pm$0.25 & -4.00$\pm$0.19\\
%             VGG-13         & 95.73$\pm$0.83  & 75.64$\pm$0.49 & 20.10$\pm$0.79 & 93.26$\pm$1.22 & 77.43$\pm$0.34  & 15.83$\pm$1.00 & -4.86$\pm$0.94\\
%             \bottomrule
%         \end{tabular}

%     \vskip -0.1in
% \end{table}

\begin{table}[h]
    \caption{Results of different architectures on 20-Split-CIFAR100(*) and 20-Split-miniImageNet($\dagger$). $\Delta$ denotes the difference between training set accuracy and testing set accuracy. All experiments are trained 5 times with 20 epoches. Learning rate is set at 0.03 and 0.01 for CIFAR and miniImageNet respectively.}
    \label{tab:archs}
    % \vskip 0.15in
    \scriptsize
    \centering
        \begin{tabular}{cccccccc}
            %\toprule
             &\multicolumn{3}{c}{\bf Single-Task Learning}  &\multicolumn{4}{c}{\bf Recursive Gradient Optimization}               \\
             %\cmidrule(r){2-4} \cmidrule(r){5-8}
             \hline
             {\bf Architecture} & {\bf ACC$_{train}$(\%)} & {\bf ACC$_{test}$(\%)}  & {\bf $\Delta$(\%)} & {\bf ACC$_{train}$(\%)} & {\bf ACC$_{test}$(\%)}  & {\bf $\Delta$(\%)}  & {\bf BWT(\%)} \\ 
            \hline
            LeNet-5$^*$ & 100.0$\pm$0.00 & 75.01$\pm$0.32 & 25.0$\pm$0.3 & 99.35$\pm$0.05 & 81.03$\pm$0.51  & 18.3$\pm$0.5 & -0.99$\pm$0.15\\
            AlexNet-6$^*$   & 99.90$\pm$0.03  & 81.60$\pm$0.31 & 18.3$\pm$0.3 & 98.78$\pm$0.08 & \textbf{82.22$\pm$0.24}& 16.6$\pm$0.2 & -1.45$\pm$0.13\\
            VGG-11$^*$         & 98.30$\pm$0.55  & 79.51$\pm$0.71 & 18.8$\pm$0.6 & 95.17$\pm$0.08 & 79.81$\pm$0.22  & 15.4$\pm$0.3 & -4.00$\pm$0.19\\
            VGG-13$^*$         & 95.73$\pm$0.83  & 75.64$\pm$0.49 & 20.1$\pm$0.8 & 93.26$\pm$1.22 & 77.43$\pm$0.34  & 15.8$\pm$1.0 & -4.86$\pm$0.94\\
            %\cmidrule(r){1-8}
            \hline 
            AlexNet-7$^\dagger$ & 99.72$\pm$0.30  & 71.92$\pm$0.55 & 27.8$\pm$0.6 & 98.10$\pm$0.11 & \textbf{72.63$\pm$0.45}  & 25.5$\pm$0.5 & -1.98$\pm$0.16\\
            VGG-11$^\dagger$    & 99.70$\pm$0.08  & 71.67$\pm$0.13 & 28.0$\pm$0.2 & 95.91$\pm$0.22 & 71.14$\pm$0.62  & 24.8$\pm$0.7 & -3.06$\pm$0.23\\
            VGG-13$^\dagger$    & 97.62$\pm$0.44  & 67.22$\pm$0.76 & 30.4$\pm$0.5 & 92.48$\pm$1.10 & 66.24$\pm$1.04  & 26.2$\pm$0.7 & -4.79$\pm$0.48\\
            ResNet-18$^\dagger$    & 97.04$\pm$0.46  & 68.78$\pm$1.09 & 28.3$\pm$0.8 & 86.79$\pm$0.85 & 71.00$\pm$0.61  & 15.8$\pm$0.3 & -5.20$\pm$0.62\\
            %\bottomrule
        \end{tabular}

    %\vskip -0.1in
\end{table}

%% file: text/discussion.tex
\section{Related work}

In this section, we present some discussions between the adopted technology with existing work. The starting point of our approach and many other loss-constrained approaches is to optimize current loss and the estimated past loss at the same time. Most of the commonly used regularization methods \citep{Kirkpatrick2017,Zenke2017,Teng2020} use a hyperparameter to balance the current task and past tasks. The objective functions of these methods can be expressed as $L(\theta) + \lambda F(\theta-\theta_{old})$. This type of method suffers from the trade-off between new tasks and old tasks and requires hyperparameter search to obtain better results. In contrast, we follow the principle of \emph{current-task-first} discarding empirical trade-offs between tasks which means $\lambda =0$. In the solution space of $\theta = \arg \min L$, we change the path of the gradient descent process through the P matrix and find the one that minimizes $F(\theta-\theta_{old})$. Thus, under the assumption of over-parameters, RGO optimizes current performance and forgetting simultaneously.

Although the starting point is different, our gradient modification process is closely related to gradient constraint methods like OWM \citep{Zeng2019a} and GPM \citep{Saha2021}. OWM uses similar iteratively updated projectors derived from \emph{recursively least square(RLS)}, 
%has been used in real-time learning since the 1970s and has been used for continual learning in recent work\citep{Zeng2019a}, 
which regards each layer as an independent linear classifier and uses the output of the previous layer to build the projection matrix. This leads to layer-by-layer accumulation of errors in modern complex end-to-end network architectures. Using the gradients of the loss function directly, our approach is less worried about the depth of the network and more compatible with existing auto-grad frameworks. For a single-layer linear classifier $y=Wx$, OWM and RGO are equivalent considering that $\frac{\partial y}{\partial W} = x$. In addition, an extra normalization procedure in our method guarantees the learning rate of the current task. This brings an additional advantage that our method can reuse the hyperparameters of original single-task models.
%%our approach can be compatible with all of the existing network and learning optimizers. 
GPM \citep{Saha2021} projects the gradient of each layer into a lower-dimensional residual space of previous tasks, while the parameter space of RGO is consistent for different tasks. RGO will maintain the network's fitting ability as the number of tasks increases. In addition, our method is not limited to the instability of SVD decomposition and does not require hyperparameter selection. 
%Our method is a current-task-first algorithm that prioritizes the learning ability and convergence rate of the current task and naturally solves the main concern of the conversion from single-task learning models to continual learning models. 

%TODO EWC 一类的方法会限制 新解的优化问题。 需要权衡 超参数

Task encoding layer has been used to reduce interference between tasks in recent works like LOS \citep{Chaudhry2020} and HAT \citep{Serra2018}, which require additional network capacity. On the contrary, FEL is only a permutation of the input corresponding to the task id. This provides an efficient task encoding and decoupling method which can be easily integrated into other continual learning methods.

\section{Limitation}
\label{sec:limitations}
Like all methods based on gradient constraint, analyses in this paper are based on neighborhood assumption and over-parameterized assumption which may not be satisfied in some narrower networks. When the number of tasks is close to the minimum number of channels in the network, this assumption fails. Although we have empirically proved that this effect is not obvious under common experimental settings, attention should be paid to the width of the network in applications. 

\section{Conclusion}

% Meanwhile, in transfer learning, same feature extractors may be a possible way to retain the knowledge of "How to learn", and learning a new feature mapper started from our random task encoded feature mapper could be a fast fine tune procedure.

% Much more work needed to let networks "understand" the relevance in different tasks rather than "have positive performance at another task" before do transfer between tasks, network need little more work to move to 1 from 0 than from 0.5 actually. Also, there is 50 percent possibility of "positive" even if tasks are totally independent.

%In this paper, we split the network structure naturally into common feature extractors and independent long-range structures.
In this paper, we propose a new recursive gradient optimization method to find the optimal parameters of fixed capacity networks, and a new feature encoding strategy to characterize the structure of the network. The feature encoding layer and the optimizer to minimize forgetting are both compatible with typical learning models, which allows our approach to be a general method to add continual learning capability into the vast majority of the existing network architectures learned by variants of gradient descent, with only constant times of memory/time cost than typical back-propagation algorithms. The theoretical derivation and experimental results show that RGO is currently the optimal approach under the current-task-first principle and quadratic loss estimation for fixed capacity networks. Experiments demonstrate that RGO achieves significantly better performance than other state-of-the-art methods on a variety of benchmarks. Without restrictions on the network structure and loss form, RGO has broad prospects in combination with other continuous learning methods and applications in other representation learning fields.

\section*{Reproducibility Statement}
We give the reproducible source code in the supplementary materials, and introduce the implementation of the baseline method in Appendix \ref{subsec:baseline}. See Appendix \ref{sec:experiment} for the selection of hyperparameters. In \verb+Python3.6+ and \verb+TensorFlow1.4+, all results can be reproduced. The theorems put forward in the main text have corresponding proofs in Appendix \ref{sec:proofs}.

%% file: text/appendix.tex
\appendix 

\section*{Appendix}
Section \ref{sec:proofs} provides proofs of theorems introduced in the paper. Section \ref{sec:implementation details} provides some details and approximations used in the implementation. Section \ref{sec:experiment} provides the details of codes, architectures, hyperparameters and resources used in the experiment.

\section{Proofs}
\label{sec:proofs}
\subsection{Proof of Theorem \ref{learningrate} } 
\label{apdx:learningrate}

As we have mentioned, the distribution of elements of the gradients of future tasks are assumed to be isotropic, and this assumption is guaranteed by our random encoding strategy under different tasks. For isotropic distribution, we have:

\begin{lemma}[Distribution Consistency]
    \label{lem:iso}
    Isotropism is invariant under orthogonal transformation. \citep{larsen2005introduction}
\end{lemma}

Then, we introduce a lemma on trace of matrix which is often used in matrix analysis \citep{horn2012matrix}: 

\begin{lemma}[Trace Consistency]
    \label{lem:trace}
    Trace of matrix is consistent under orthogonal transformation. 
\end{lemma}

As described in Section \ref{sec:optimizer}, at $i$-th single train step during task $k$, we have:

\begin{equation}
    \left\{
    \begin{aligned}
&\theta_{i} = \theta_{i-1} -\eta_{i} P \nabla L_{k}(\theta_{i-1}) \\
&L_{k}(\theta_i) = L_k(\theta_{i-1}) -\eta_{i}(\nabla L_{k}(\theta_{i-1}))^T P\nabla L_{k}(\theta_{i-1})
    \end{aligned}
    \right.
   \end{equation}

   For any potential gradient, the mathematical expectation of the loss function decline is:

\begin{equation}
    \mathbb{E}[L_{k}(\theta_i)-L_k(\theta_{i-1})]=-\eta \mathbb{E}[(\nabla L_{k})^T P \nabla L_{k}]
\end{equation}

the optimizer degrades to original single task optimizer when $P$ equals to identify matrix. As positive definite symmetric matrix can be orthogonal diagonalized by a orthogonal matrix $V$,

\begin{equation}
    VPV^T = diag\{\lambda_1,\lambda_2,\cdots,\lambda_n\} :=\Lambda
\end{equation}

where $\{\lambda_i\}$ represent the eigen values of the Projection Matrix $P$. Mark $dim(P)$ and $\nabla L_{k}(\theta)$ with $n$ and $x=(x_1,x_2,\cdots,x_n)$ respectively. 
We assume the distribution of unknown future gradients is isotropic and apply Lemma \ref{lem:iso} to the expectation:

\begin{equation}
    \begin{aligned}
    &\mathbb{E}[(\nabla L_{k})^T P \nabla L_{k}] =\mathbb{E}_{x}[x^T V^T \Lambda_P V x] =\mathbb{E}_{x}[x^T \Lambda_P x]\\
    &=\sum_{i=1}^n \lambda_i \mathbb{E}_{x}[x_i^2]=\sum_{i=1}^n \lambda_i \frac{ \mathbb{E}_{x}[\sum_{i=1}^n x_i^2]}{n}\\
    &= \frac{\sum_{i=1}^n \lambda_i}{n} \mathbb{E}_{ x}[x^T x] = \frac{\sum_{i=1}^n \lambda_i}{n} \mathbb{E}[(\nabla L_{k})^T \nabla L_{k}]
    \end{aligned}
\end{equation}

According to Lemma \ref{lem:trace}, the sum of the eigen values of Inverse Hessian Matrix $P$ equals to that of $\Lambda$, that means:

\begin{equation}
    \sum_{i=1}^n \lambda_i = trace(P)
\end{equation}

Thus if we add a normalize a constraint ${trace(P)=dim(P)}$ to the Inverse Hessian matrix $P$, every task in the continual learning procedure can have consistent expectation convergence rate.

% \begin{theorem}[consistent convergence]
%     \label{convergence}
%     For any possible gradients, the mathematical expectation of the loss decrease:
%     \begin{equation*}
%         E[(\nabla L_{t+1})^T P \nabla L_{t+1}] = \frac{trace(P)}{dim(P)} E[(\nabla L_{t+1})^T \nabla L_{t+1}]
%     \end{equation*}
%     if $P$ is positive symmetric matrix.
% \end{theorem}

\subsection{Proof of loss function equivalence}
\label{apdx:RLL}

\begin{theorem}[Loss Equivalence]
    If $\theta_{k-1}^* = \arg\min_\theta F_{k-1}^{RLL}(\theta) $, using $F_{k}^{RLL}$ or $F_{k}$ as the loss to optimize is equivalent, which means $\arg\min_\theta F_{k}^{RLL}(\theta)=\arg\min_\theta F_{k}(\theta)$
\end{theorem}

We first expand $F_k$ in Equation \ref{F} at the initial state ${\theta_{k-1}^*}$ during $\mathcal{T}_k$:

\begin{equation}
    \begin{aligned}
    &F_k(\theta) \approx \sum_{j=1}^{k-1}[L_j(\theta_j^*)+ \frac{1}{2}(\theta-\theta_{k-1}^*+\theta_{k-1}^* -\theta_{j}^*)^T H_j (\theta-\theta_{k-1}^*+\theta_{k-1}^* -\theta_{j}^*)] \\
    %&= \sum_{j=1}^{k-1}[L_j(\theta_j^*)+ \frac{1}{2}(\theta_{k-1}^* -\theta_{j}^*)^T H_j (\theta_{k-1}^* -\theta_{j}^*) + (\theta_k^*-\theta_{k-1}^*)^TH_j (\theta_{k-1}^* -\theta_{j}^*) + \frac{1}{2}(\theta_k^*-\theta_{k-1}^*)^T H_j(\theta_k^*-\theta_{k-1}^*)] \\
    &= (\theta-\theta_{k-1}^*)^T\sum_{j=1}^{k-2}[H_j (\theta_{k-1}^* -\theta_{j}^*)]+ \frac{1}{2}(\theta-\theta_{k-1}^*)^T (\sum_{j=1}^{k-1}H_j)(\theta-\theta_{k-1}^*) + const.
    \end{aligned}
\end{equation}

For further analysis, we first introduce the following lemma which can be proved inductively.

\begin{lemma}
    \label{lemma1}
    if $(\sum_{j=1}^{k-1}H_j)(\theta_k^*-\theta_{k-1}^*) = 0$ holds for any $k$, then $\sum_{j=1}^{k-1}[H_j(\theta_k^*-\theta_j^*)] = 0$ also holds for any $k$.
\end{lemma}

\begin{proof}
    We mark $\sum_{j=1}^{k-1}[H_j(\theta_k^*-\theta_j^*)] $ as $D(k)$, then for $k = 1,2,3,\cdots$ we have 

    \begin{equation*}
        \begin{aligned}
        &D(k) - D(k-1)  \\ 
        &= \sum_{j=1}^{k-1}[H_j(\theta_k^*-\theta_j^*)] - \sum_{j=1}^{k-2}[H_j(\theta_{k-1}^*-\theta_j^*)] \\
        &= H_{k-1}(\theta_k^*-\theta_{k-1}^*)  + (\sum_{j=1}^{k-2}H_j)(\theta_k^*-\theta_{k-1}^*)\\
        &= (\sum_{j=1}^{k-1}H_j)(\theta_k^*-\theta_{k-1}^*)\\
        &= 0
        \end{aligned}
    \end{equation*}

    Using the fact that $D(0) = 0$ , $D(k) = 0$ holds for all positive integer $k$.
\end{proof}

Ignoring the constant terms, the key to prove this theorem is the second term of the expression of $F_k$;

\begin{proof}
    If ${\nabla F_k^{RLL}(\theta_k^*)=0}$ satisfies for all $k\in [K]$, then we can get
    \[
        (\sum_{j=1}^{k-1}H_j)(\theta_k^*-\theta_{k-1}^*) = 0 \ , \forall k\in [K]
    \]
    Using the result of Lemma \ref{lemma1}, we have 
    \[
        \sum_{j=1}^{k-2}[H_j (\theta_{k-1}^* -\theta_{j}^*)]=0 \ , \forall k\in [K]
    \]
    Substituting this formula into Equation (\ref{F}) and discarding the constant terms lead to
    \[
        F_k(\theta_k^*) = F_k^{RLL}(\theta_k^*) \ , \forall k\in [K]
    \]
    Thus, $\{F_k^{RLL}\}_{k=1}^K$ and $\{F_k\}_{k=1}^K$ are equivalent loss sequences throughout the continual learning process.
\end{proof}

Note that this equivalence is obtained in the case of a second-order approximation of the loss functions of tasks, so it also depends on the "close vicinity" hypothesis.

\subsection{Proof of Theorem \ref{upperbound}}
\label{apdx:upperbound}
\begin{lemma}[Cauchy Inequality]
    \label{lem:ineq}
    For any positive integer $n$, positive symmetric definite matrix $P$ and vectors $\{v_i\}_{i=1}^n$, denote the $P$-norm of $v$ by $||v||_P = \sqrt{v^TPv}$, we have: 
    \[
        ||\sum_{i=1}^n v_i||_P^2 \leq n \sum_{i=1}^n ||v_i||_P^2 \]
\end{lemma}

    Denote $\hat{\sigma}_{m}(\cdot)$ as the symbol for finding maximum eigenvalue and $\eta_m$ as the maximum single-step learning rate,
     the recursive least loss has an upper bound:

    % 上界的证明过程
\begin{equation}
    \begin{aligned}
    &F_k^{RLL}(\theta_k^*) =\frac{1}{2} (\sum_{i=1}^{n_k} \eta_{i}\nabla L_{k}(\theta_{i-1}))^T P(\sum_{j=1}^{k-1}H_j)P (\sum_{i=1}^{n_k} \eta_{i}\nabla L_{k}(\theta_{i-1})) \\
    &\leq \hat{\sigma}_{m}[P(\sum_{j=1}^{k-1}H_j)]\frac{1}{2} (\sum_{i=1}^{n_k} \eta_{i}\nabla L_{k}(\theta_{i-1}))^T P (\sum_{i=1}^{n_k} \eta_{i}\nabla L_{k}(\theta_{i-1})) \\
\end{aligned}
\end{equation}

According to Lemma \ref{lem:ineq}, we have

\begin{equation}
    \label{equ:app1}
    \begin{aligned}
    &F_k^{RLL}(\theta_k^*) \leq \hat{\sigma}_{m}[P(\sum_{j=1}^{k-1}H_j)]\frac{1}{2}n_k \sum_{i=1}^{n_k} (\eta_{i}\nabla L_{k}(\theta_{i-1}))^T P (\eta_{i}\nabla L_{k}(\theta_{i-1})) \\
    &\leq \frac{1}{2}n_k\eta_m\hat{\sigma}_{m}[P(\sum_{j=1}^{k-1}H_j)] \sum_{i=1}^{n_k} \eta_{i}(\nabla L_{k}(\theta_{i-1}))^T P \nabla L_{k}(\theta_{i-1}) \\
    \end{aligned}
\end{equation}

Considering that $L_{k}(\theta_{k}^*)\geq 0$ at the end of training and loss for current task before training can be expressed as $L_{k}(\theta_{k-1}^*)$, the change of loss function described in Equation \ref{equ:totalupdate} satisfies the following inequality:

    \begin{equation}
        \label{equ:app2}
        L_{k}(\theta_{k-1}^*) \geq \sum_{i=1}^{n_k} \eta_{i}(\nabla L_{k}(\theta_{i-1}))^T P\nabla L_{k}(\theta_{i-1})
    \end{equation}

Then we can complete the proof by combining \ref{equ:app1} with \ref{equ:app2}.

\begin{equation}
    F_k^{RLL}(\theta_k^*)\leq \frac{1}{2}n_k\eta_m\hat{\sigma}_{m}[P(\sum_{j=1}^{k-1}H_j)]L_k(\theta_{k-1}^*)\\
\end{equation}

\subsection{Solution of Problem \ref{problem:optimization}}
\label{apdx:solution}

At a fixed point of the model, total Hessian matrix $\bar{H}:=\sum_{j=1}^{k-1}H_j$ must be positive definite like $P$. As positive definite symmetric matrix can be orthogonal diagonalized by an orthogonal matrix \citep{horn2012matrix}, 
 $H$ and $P$ can be expressed as $\bar{H}=U^T \Lambda_{\bar{H}} U$, $P=V^T \Lambda_P V$, 
 while ${\Lambda_{\bar{H}}=\{\sigma_1,\cdots,\sigma_n\}}$ and $\Lambda_P=\{\lambda_1,\cdots,\lambda_n\}$ are diagonal matrices.

% Thus the expected convergence rate of the current task is maintained at same level with single task learner, Thus FMEN can be a general method to add memory ability to any existing model with same complexity with back-propagation, with network fitting power/convergence rate remaining nearly same.

% Actually, different learning rate hardlly affect the performance on holding the learned memory,
%  FMEN more likely provide a direction, after new task fully trained, the optimum parameter hardlly affected by learning rate, the influence on previous task is similiar. Thus, we prefer ensuring the learning rate of new task, rather than saving the old task information better, witch is closer to the actual application scenario.

% Unlike parameter isolation based methods reduce the number of trainable parameters / gradient based method constrain the current gradient to a subspace ,/regularization based
%  method add extra loss, FMEN won't weaken the fitting ability of the network for current task as trainable paremeters/convergence rate is always same.

Using the diagonalization step, the optimization problem \ref{problem:optimization} can be expressed as:

\begin{equation}
    P: \left\{
    \begin{aligned}
    &\mathop {\min }\limits_P \quad \hat{\sigma}_{m}(V^T \Lambda_P V U^T \Lambda_{\bar{H}} U)  \\
    &\mathop {s.t.}\quad trace(P)=dim(P)
    \end{aligned}
    \right.
\end{equation}

As orthogonal transformation maintains eigen value \citep{horn2012matrix}, which means:

\[
    \hat{\sigma}_{m}(V^T \Lambda_P V U^T \Lambda_{\bar{H}} U) = \hat{\sigma}_{m}( UV^T \Lambda_P V U^T \Lambda_{\bar{H}}) \]

% TODO 应该有更严格的证明来的。    
There are two main variable to be optimized, diagonal matrix $\Lambda_P$ and orthogonal matrix $UV^T$. To simplify the derivation, we set $UV^T$ to the simplest orthogonal matrix $I$. Under this assumption, the optimization problem is simplified as:

\begin{equation}
    \{\lambda_i\}_{i=1}^n: \left\{
    \begin{aligned}
    &\mathop {\min }\limits_{\lambda_i} \quad \max_i \sigma_i \lambda_i  \\
    &\mathop {s.t.}\quad \sum_{i=1}^n\lambda_i=n
    \end{aligned}
    \right.
\end{equation}

We get the optimal eigen values:

\begin{equation}
    \lambda_i=\frac{n}{\sigma_i\sum\frac{1}{\sigma_i}}  
\end{equation}

Thus,

\begin{equation}
    P=V\Lambda_PV^T 
    = \frac{n}{\sum\frac{1}{\sigma_i}} V \Lambda_{\bar{H}}^{-1} V^T 
    =\frac{dim(\bar{H})}{trace(\bar{H}^{-1})}\bar{H}^{-1}
\end{equation}

\subsection{Proof of Theorem \ref{localequivalence}}

Denote the gradient of total parameter set and the total Hessian matrix as $g_\theta=(g_1^T,g_2^T,\cdots,g_L^T)^T$ and $\{\frac{\partial ^2 L_k}{\partial \theta^2}\}_{l,r} = \frac{\partial ^2 L_k}{\partial h_l \partial h_r}$, we have:

\label{apdx:localequivalence}
\begin{equation}
    \begin{aligned}
    g_l^TP_l \frac{\partial ^2 L_k}{\partial h_l \partial h_r}P_r g_r &= g_l^T V_l^T \Lambda_{P_l} V_l  U_l^T \Lambda_{\bar{H}_l}^{\frac{1}{2}} \Lambda_{\bar{H}_r}^{\frac{1}{2}} U_r  V_r^T \Lambda_{P_r} V_r  g_r \\ 
    &= g_l^T V_l^T   \Lambda_{P_l}^{\frac{1}{2}}\Lambda_{P_l} ^{\frac{1}{2}}\Lambda_{\bar{H}_l}^{\frac{1}{2}} \Lambda_{\bar{H}_r}^{\frac{1}{2}}\Lambda_{P_r} ^{\frac{1}{2}} \Lambda_{P_r} ^{\frac{1}{2}}  V_r  g_r \\ 
    &\leq \hat{\sigma}_m(\Lambda_{P_l} ^{\frac{1}{2}}\Lambda_{\bar{H}_l}^{\frac{1}{2}} ) \hat{\sigma}_m(\Lambda_{\bar{H}_r}^{\frac{1}{2}}\Lambda_{P_r} ^{\frac{1}{2}}) g_l^T V_l^T   \Lambda_{P_l}^{\frac{1}{2}} \cdot  \Lambda_{P_r} ^{\frac{1}{2}}  V_r  g_r 
\end{aligned}
\end{equation}

\begin{equation}
    \begin{aligned}
    g_\theta^T \frac{\partial ^2 L_k}{\partial \theta^2} g_\theta&= \sum_{1\leq l,r\leq L}g_l^TP_l \frac{\partial ^2 L_k}{\partial h_l \partial h_r}P_r g_r \\
    &\leq \sum_{1\leq l,r\leq L} \hat{\sigma}_m(\Lambda_{P_l} ^{\frac{1}{2}}\Lambda_{\bar{H}_l}^{\frac{1}{2}} ) \hat{\sigma}_m(\Lambda_{\bar{H}_r}^{\frac{1}{2}}\Lambda_{P_r} ^{\frac{1}{2}}) g_l^T V_l^T   \Lambda_{P_l}^{\frac{1}{2}} \cdot  \Lambda_{P_r} ^{\frac{1}{2}}  V_r  g_r \\
    &\leq \text{max}_{l} [\hat{\sigma}_m(\Lambda_{P_l} \Lambda_{\bar{H}_l} )] \sum_{1\leq l,r\leq L}  g_l^T V_l^T   \Lambda_{P_l}^{\frac{1}{2}} \cdot  \Lambda_{P_r} ^{\frac{1}{2}}  V_r  g_r \\
    &= \text{max}_{l} [\hat{\sigma}_m(\Lambda_{P_l} \Lambda_{\bar{H}_l} )]  g_\theta^T P g_\theta
\end{aligned}
\end{equation}

The optimization problem above has same form as the global optimization problem in Section \ref{apdx:solution}, the solution 
can be easily got as:

\[
    P_l=\frac{dim(\bar{H_l})}{trace(\bar{H_l}^{-1})}\bar{H_l}^{-1} ,where \quad \bar{H_l}=\sum_{j=1}^{k-1} H_{j,l}
\]

\section{Implementation details}
\label{sec:implementation details}
\subsection{Details of quadratic estimation of the Hessian matrix}
\label{apdx:quadratic}

For the $C$-class classification problems, $f(x;\theta)$ has $C$-logits associated to different classes. We consider the most commonly used softmax cross entropy loss which is defined as
\begin{equation}
    l(y,f(x;\theta)) = -\sum_{j=1}^c y_j log(a_j)
\end{equation}
where $a_j= exp(f_j(x;\theta))/\sum_{c=1}^C exp(f_c(x;\theta))$ as the $j$-th softmax output. The $(i,j)$-th element of the second derivative matrix of the loss function with respect to $f(x;\theta)$ is then calculated as
\begin{equation}
    l''(y;f(x,\theta))_{i,j} = a_j \phi_{i,j} - a_ia_j
\end{equation}
where $\phi_{i,j}$ is Dirac function equal to $1$ while $i=j$ else $0$.

As a symmetric diagonally dominant matrix, $l''(y;f(x,\theta))$ has its matrix root $[l''(y;f(x,\theta))]^{\frac{1}{2}}$. This guarantees the correctness of our algorithm. In implementation, for convenience, we only used the diagonal element corresponding to the ground truth label for an estimation.

\subsection{Time \& memory complexity of RGO}

We list the shape of projection matrix and time complexity of projection matrix update and gradient modification introduced by RGO for some typical feature extractors below:

\label{apdx:complexity}
\begin{table}[ht]
    \centering
    \begin{tabular}{lrrrrrr}
    {\bf Kind} &{\bf Shape }& {\bf Size of P}  &{\bf time complexity} &\\
    \hline
    vector & $n_1$  &  $(n_1,n_1)$ & $n_1^2$ &  \\
    matrix & $(n_1,n_2)$  & $ (n_1,n_1) $ & $n_1^2$ & \\
    kernel & $(n_1,n_2,ksize,ksize)$ & $ (ksize^2n_1,ksize^2n_1) $ &$ksize^4n_1^2$ & \\
    % Ours & \textbf{92.6} & \textbf{71.1}  & \textbf{79.35} &\textbf{93.2}& \textbf{90.25} & \textbf{85.3} \\
    \end{tabular}
    %\caption{accuracy of each task after trained sequentially, Rotated MNIST}
    %\caption{single step complexity of the local leaner}
    \label{tab:complexity}
\end{table}

First, according to Algorithm \ref{alg:algorithm}, the time complexity of updating $P$ is obviously O($\text{dim}(P)^2$). The main concern comes from the matrix-matrix product in $g'=Pg$ for $g$s with higher dimension. However, if we notice the linear correlation of the columns of $g$, we can avoid this matrix multiplication. Use a fully connected layer $y= xW+b: x\in \mathbb{R}^{n_1}, y\in \mathbb{R}^{n_2}$ as an example. Considering $\frac{\partial L}{\partial W} = \frac{\partial L}{\partial y} x^T$, we have $g' = Pg = P\frac{\partial L}{\partial y} x^T$. If we calculate from left to right instead of calculating $g$ first, we can avoid matrix multiplication and reduce the number of calculations from $n_1n_2+n_1^2n_2$ to $n_1n_2+n_1^2$. The amount of calculation beyond the original backpropagation is only $n_1^2$. The calculation process for kernels is the same except for a reshape process.

Considering that both the kernel size and $\frac{n_1}{n_2}$ have upper bounds in common neural network models, the time complexity of RGO remains the same as that of backpropagation.

\section{Experiment details}
\label{sec:experiment}
\subsection{Baseline implementations}
\label{subsec:baseline}
EWC \citep{Kirkpatrick2017}, LOS \citep{Chaudhry2020}, A-GEM \citep{Chaudhry2019}, and ER-ring \citep{Chaudhry2019b} are implemented from adapting the code provided by \citet{Chaudhry2020} under MIT License. GPM is implemented from the official implementation provided by \citet{Saha2021} under MIT License. 

\subsection{Resources}
\label{sec:resources}
All experiments of our method are completed in several hours with 4 pieces of Nvidia-2080Ti GPUs.

\subsection{Architectures}
We provide details of architectures we used in the experiment section.
\label{apdx:archs}
\begin{itemize}
    \item LeNet-5: A modified LeNet used by \citet{Yoon2019}. There are two convolutional layers with kernels size of (5,5) and channels of (20,50), followed by two hidden fully connected layer with (800,500) units.
    \item AlexNet-6: A modified AlexNet used by \citet{Saha2021}. There are three convolutional layers with kernels size of (4,3,2) and channels of (64,128,256), followed by two hidden fully connected layer with (2048,2048) units.
    \item AlexNet-7: A modified AlexNet. There are four convolutional layers with kernels size of (5,4,3,3) and channels of (64,128,128,128), followed by two hidden fully connected layer with (2048,2048) units.
    \item VGG-11\&VGG-13 : Original VGG11 and VGG13 proposed by \citet{Simonyan2015a}.
    \item ResNet-18: A standard 18-layer ResNet proposed by \citet{He2016}. For our approach, we remove all batch-norm layers because their parameters are not updated by gradient descent.
\end{itemize}
LeNet-like and AlexNet-like architectures are attached a 2$\times$2 maxpooling layer after each convolutional layer.

\subsection{Hyperparameters}
\label{apdx:hyperparameter}

The learning rates of all baselines are generated by hyperparameter search in [0.003,0.01,0.03,0.1,0.3,1] to achieve better results. Other hyperparameters of EWC, A-GEM, ER-Ring and LOS follows \citet{Chaudhry2020}, while those of GPM and APD follows their official implementation.

\begin{itemize}
    \item Single Task Learning\begin{itemize}
        \item learningrate: 0.1(MNIST), 0.03(CIFAR100, miniImageNet)
    \end{itemize}
    \item Recursive Gradient Optimization(Ours)\begin{itemize}
        \item learningrate: 0.1(MNIST), 0.03(CIFAR100, miniImageNet 2000steps), 0.01(miniImageNet 20epochs)
    \end{itemize}
    \item SGD\begin{itemize}
        \item learningrate: 0.1(MNIST), 0.03(CIFAR100, miniImageNet)
    \end{itemize}
    \item EWC\begin{itemize}
        \item learningrate: 0.1(MNIST), 0.03(CIFAR100, miniImageNet)
        \item regularization: 10(MNIST, CIFAR100, miniImageNet)
    \end{itemize}
    \item A-GEM\begin{itemize}
        \item learningrate: 0.1(MNIST), 0.03(CIFAR100, miniImageNet)
    \end{itemize}
    \item ER-Ring\begin{itemize}
        \item learningrate: 0.1(MNIST), 0.03(CIFAR100, miniImageNet)
    \end{itemize}
    \item LOS\begin{itemize}
        \item learningrate: 0.1(MNIST), 0.4(CIFAR100), 0.2(miniImageNet)
    \end{itemize}
    \item GPM\begin{itemize}
        \item learningrate: 0.1(MNIST), 0.03(CIFAR100, miniImageNet)
        \item threshold: 0.95 for first layer and 0.99 for other layers(MNIST) ,increase from 0.97 to 1(CIFAR), increase from 0.985 to 1(miniImageNet)
        \item dimension of representation matrices: 300(MNIST), 125(CIFAR), 100(miniImageNet)
    \end{itemize}

\end{itemize}

%% file: RGO.bbl
\begin{thebibliography}{38}
\providecommand{\natexlab}[1]{#1}
\providecommand{\url}[1]{\texttt{#1}}
\expandafter\ifx\csname urlstyle\endcsname\relax
  \providecommand{\doi}[1]{doi: #1}\else
  \providecommand{\doi}{doi: \begingroup \urlstyle{rm}\Url}\fi

\bibitem[Alet et~al.(2018)Alet, Lozano-P{\'{e}}rez, and Kaelbling]{Alet2018}
Ferran Alet, Tom{\'{a}}s Lozano-P{\'{e}}rez, and Leslie~P. Kaelbling.
\newblock {Modular meta-learning}.
\newblock \penalty0 (CoRL), 2018.
\newblock URL \url{http://arxiv.org/abs/1806.10166}.

\bibitem[Aljundi et~al.(2017)Aljundi, Chakravarty, and Tuytelaars]{Aljundi2017}
Rahaf Aljundi, Punarjay Chakravarty, and Tinne Tuytelaars.
\newblock {Expert gate: Lifelong learning with a network of experts}.
\newblock \emph{Proceedings - 30th IEEE Conference on Computer Vision and
  Pattern Recognition, CVPR 2017}, 2017-January:\penalty0 7120--7129, 2017.
\newblock \doi{10.1109/CVPR.2017.753}.

\bibitem[Amari(1998)]{amari1998natural}
Shun-Ichi Amari.
\newblock Natural gradient works efficiently in learning.
\newblock \emph{Neural computation}, 10\penalty0 (2):\penalty0 251--276, 1998.

\bibitem[Azizan et~al.(2019)Azizan, Lale, and Hassibi]{Azizan2019a}
Navid Azizan, Sahin Lale, and Babak Hassibi.
\newblock {Stochastic Mirror Descent on Overparameterized Nonlinear Models:
  Convergence, Implicit Regularization, and Generalization}.
\newblock pp.\  1--35, 2019.
\newblock URL \url{http://arxiv.org/abs/1906.03830}.

\bibitem[Chang et~al.(2019)Chang, Levine, Gupta, and Griffiths]{Chang2019}
Michael~B. Chang, Sergey Levine, Abhishek Gupta, and Thomas~L. Griffiths.
\newblock {Automatically composing representation transformations as a means
  for generalization}.
\newblock \emph{7th International Conference on Learning Representations, ICLR
  2019}, pp.\  1--23, 2019.

\bibitem[Chaudhry et~al.(2019{\natexlab{a}})Chaudhry, Marc'Aurelio, Rohrbach,
  and Elhoseiny]{Chaudhry2019}
Arslan Chaudhry, Ranzato Marc'Aurelio, Marcus Rohrbach, and Mohamed Elhoseiny.
\newblock {Efficient lifelong learning with A-GEM}.
\newblock \emph{7th International Conference on Learning Representations, ICLR
  2019}, pp.\  1--20, 2019{\natexlab{a}}.

\bibitem[Chaudhry et~al.(2019{\natexlab{b}})Chaudhry, Rohrbach, Elhoseiny,
  Dokania, Torr, Ajanthan, and Ranzato]{Chaudhry2019b}
Arslan Chaudhry, Marcus Rohrbach, Mohamed Elhoseiny, Puneet~K. Dokania,
  Philip~H.S. Torr, Thalaiyasingam Ajanthan, and Marc'Aurelio Ranzato.
\newblock {On tiny episodic memories in continual learning}.
\newblock \emph{arXiv}, pp.\  1--15, 2019{\natexlab{b}}.
\newblock ISSN 23318422.

\bibitem[Chaudhry et~al.(2020)Chaudhry, Khan, Dokania, and Torr]{Chaudhry2020}
Arslan Chaudhry, Naeemullah Khan, Puneet~K. Dokania, and Philip H.~S. Torr.
\newblock {Continual Learning in Low-rank Orthogonal Subspaces}.
\newblock \penalty0 (NeurIPS):\penalty0 1--12, 2020.
\newblock URL \url{http://arxiv.org/abs/2010.11635}.

\bibitem[Farajtabar et~al.(2019)Farajtabar, Azizan, Mott, and
  Li]{Farajtabar2019}
Mehrdad Farajtabar, Navid Azizan, Alex Mott, and Ang Li.
\newblock {Orthogonal Gradient Descent for Continual Learning}.
\newblock 2019.
\newblock URL \url{http://arxiv.org/abs/1910.07104}.

\bibitem[Fernando et~al.(2017)Fernando, Banarse, Blundell, Zwols, Ha, Rusu,
  Pritzel, and Wierstra]{Fernando2017}
Chrisantha Fernando, Dylan Banarse, Charles Blundell, Yori Zwols, David Ha,
  Andrei~A. Rusu, Alexander Pritzel, and Daan Wierstra.
\newblock {PathNet: Evolution Channels Gradient Descent in Super Neural
  Networks}.
\newblock 2017.
\newblock URL \url{http://arxiv.org/abs/1701.08734}.

\bibitem[Goodfellow et~al.(2014)Goodfellow, Mirza, Xiao, Courville, and
  Bengio]{Goodfellow2014}
Ian~J. Goodfellow, Mehdi Mirza, Da~Xiao, Aaron Courville, and Yoshua Bengio.
\newblock {An empirical investigation of catastrophic forgetting in
  gradient-based neural networks}.
\newblock \emph{2nd International Conference on Learning Representations, ICLR
  2014 - Conference Track Proceedings}, 2014.

\bibitem[Haykin(2002)]{Simon2002}
Simon Haykin.
\newblock \emph{Adaptive Filter Theory}.
\newblock Prentice Hall, 2002.

\bibitem[He et~al.(2016)He, Zhang, Ren, and Sun]{He2016}
Kaiming He, Xiangyu Zhang, Shaoqing Ren, and Jian Sun.
\newblock {Deep residual learning for image recognition}.
\newblock \emph{Proceedings of the IEEE Computer Society Conference on Computer
  Vision and Pattern Recognition}, 2016-December:\penalty0 770--778, 2016.
\newblock ISSN 10636919.
\newblock \doi{10.1109/CVPR.2016.90}.

\bibitem[Horn \& Johnson(2012)Horn and Johnson]{horn2012matrix}
Roger~A Horn and Charles~R Johnson.
\newblock \emph{Matrix analysis}.
\newblock Cambridge university press, 2012.

\bibitem[Kaushik et~al.(2021)Kaushik, Gain, Kortylewski, and
  Yuille]{Kaushik2021}
Prakhar Kaushik, Alex Gain, Adam Kortylewski, and Alan Yuille.
\newblock {Understanding Catastrophic Forgetting and Remembering in Continual
  Learning with Optimal Relevance Mapping}.
\newblock 2021.
\newblock URL \url{http://arxiv.org/abs/2102.11343}.

\bibitem[Kirkpatrick et~al.(2017)Kirkpatrick, Pascanu, Rabinowitz, Veness,
  Desjardins, Rusu, Milan, Quan, Ramalho, Grabska-Barwinska, Hassabis, Clopath,
  Kumaran, and Hadsell]{Kirkpatrick2017}
James Kirkpatrick, Razvan Pascanu, Neil Rabinowitz, Joel Veness, Guillaume
  Desjardins, Andrei~A. Rusu, Kieran Milan, John Quan, Tiago Ramalho, Agnieszka
  Grabska-Barwinska, Demis Hassabis, Claudia Clopath, Dharshan Kumaran, and
  Raia Hadsell.
\newblock {Overcoming catastrophic forgetting in neural networks}.
\newblock \emph{Proceedings of the National Academy of Sciences of the United
  States of America}, 114\penalty0 (13):\penalty0 3521--3526, 2017.
\newblock ISSN 10916490.
\newblock \doi{10.1073/pnas.1611835114}.

\bibitem[Larsen \& Marx(2005)Larsen and Marx]{larsen2005introduction}
Richard~J Larsen and Morris~L Marx.
\newblock \emph{An introduction to mathematical statistics}.
\newblock Prentice Hall, 2005.

\bibitem[LeCun(1998)]{Lecun1998}
Yann LeCun.
\newblock {The mnist database of handwritten digits}.
\newblock 1998.
\newblock URL \url{http://yann.lecun.com/exdb/mnist/}.

\bibitem[Li et~al.(2019)Li, Zhou, Wu, Socher, and Xiong]{Li2019}
Xilai Li, Yingbo Zhou, Tianfu Wu, Richard Socher, and Caiming Xiong.
\newblock {Learn to Grow: A Continual Structure Learning Framework for
  Overcoming Catastrophic Forgetting}.
\newblock 2019.
\newblock URL \url{http://arxiv.org/abs/1904.00310}.

\bibitem[Lopez-Paz \& Ranzato(2017)Lopez-Paz and Ranzato]{Lopez-Paz2017}
David Lopez-Paz and Marc'Aurelio Ranzato.
\newblock {Gradient episodic memory for continual learning}.
\newblock \emph{Advances in Neural Information Processing Systems},
  2017-December\penalty0 (Nips):\penalty0 6468--6477, 2017.
\newblock ISSN 10495258.

\bibitem[Nguyen et~al.(2018)Nguyen, Li, Bui, and Turner]{Nguyen2018}
Cuong~V. Nguyen, Yingzhen Li, Thang~D. Bui, and Richard~E. Turner.
\newblock {Variational continual learning}.
\newblock \emph{6th International Conference on Learning Representations, ICLR
  2018 - Conference Track Proceedings}, \penalty0 (Vi):\penalty0 1--18, 2018.

\bibitem[Ritter et~al.(2018)Ritter, Botev, and Barber]{Ritter2018}
Hippolyt Ritter, Aleksandar Botev, and David Barber.
\newblock {Online structured laplace approximations for overcoming catastrophic
  forgetting}.
\newblock \emph{Advances in Neural Information Processing Systems},
  2018-December:\penalty0 3738--3748, 2018.
\newblock ISSN 10495258.

\bibitem[Rosenbaum et~al.(2018)Rosenbaum, Klinger, and Riemer]{Rosenbaum2018}
Clemens Rosenbaum, Tim Klinger, and Matthew Riemer.
\newblock {Routing networks: Adaptive selection of non-linear functions for
  multi-task learning}.
\newblock \emph{6th International Conference on Learning Representations, ICLR
  2018 - Conference Track Proceedings}, pp.\  1--16, 2018.

\bibitem[Russakovsky et~al.(2015)Russakovsky, Deng, Su, Krause, Satheesh, Ma,
  Huang, Karpathy, Khosla, Bernstein, Berg, and Fei-Fei]{Russakovsky2015}
Olga Russakovsky, Jia Deng, Hao Su, Jonathan Krause, Sanjeev Satheesh, Sean Ma,
  Zhiheng Huang, Andrej Karpathy, Aditya Khosla, Michael Bernstein,
  Alexander~C. Berg, and Li~Fei-Fei.
\newblock {ImageNet Large Scale Visual Recognition Challenge}.
\newblock \emph{International Journal of Computer Vision}, 115\penalty0
  (3):\penalty0 211--252, 2015.
\newblock ISSN 15731405.
\newblock \doi{10.1007/s11263-015-0816-y}.
\newblock URL \url{http://dx.doi.org/10.1007/s11263-015-0816-y}.

\bibitem[Rusu et~al.(2016)Rusu, Rabinowitz, Desjardins, Soyer, Kirkpatrick,
  Kavukcuoglu, Pascanu, and Hadsell]{Rusu2016}
Andrei~A. Rusu, Neil~C. Rabinowitz, Guillaume Desjardins, Hubert Soyer, James
  Kirkpatrick, Koray Kavukcuoglu, Razvan Pascanu, and Raia Hadsell.
\newblock {Progressive Neural Networks}.
\newblock 2016.
\newblock URL \url{http://arxiv.org/abs/1606.04671}.

\bibitem[Saha et~al.(2021)Saha, Garg, and Roy]{Saha2021}
Gobinda Saha, Isha Garg, and Kaushik Roy.
\newblock {Gradient Projection Memory for Continual Learning}.
\newblock \penalty0 (2018):\penalty0 1--18, 2021.
\newblock URL \url{http://arxiv.org/abs/2103.09762}.

\bibitem[Serra et~al.(2018)Serra, Suris, Mir{\'{o}}n, and
  Karatzoglou]{Serra2018}
Joan Serra, D{\'{i}}dac Suris, Marius Mir{\'{o}}n, and Alexandras Karatzoglou.
\newblock {Overcoming Catastrophic forgetting with hard attention to the task}.
\newblock \emph{35th International Conference on Machine Learning, ICML 2018},
  10:\penalty0 7225--7234, 2018.
\newblock ISSN 2640-3498.

\bibitem[Shen et~al.(2020)Shen, Zhang, Chen, and Deng]{Shen2020}
Gehui Shen, Song Zhang, Xiang Chen, and Zhi~Hong Deng.
\newblock {Generative Feature Replay with Orthogonal Weight Modification for
  Continual Learning}.
\newblock \emph{arXiv}, 2020.

\bibitem[Shin et~al.(2017)Shin, Lee, Kim, and Kim]{Shin2017}
Hanul Shin, Jung~Kwon Lee, Jaehong Kim, and Jiwon Kim.
\newblock {Continual learning with deep generative replay}.
\newblock \emph{Advances in Neural Information Processing Systems},
  2017-December\penalty0 (Nips):\penalty0 2991--3000, 2017.
\newblock ISSN 10495258.

\bibitem[Simonyan \& Zisserman(2015)Simonyan and Zisserman]{Simonyan2015a}
Karen Simonyan and Andrew Zisserman.
\newblock {Very deep convolutional networks for large-scale image recognition}.
\newblock \emph{3rd International Conference on Learning Representations, ICLR
  2015 - Conference Track Proceedings}, pp.\  1--14, 2015.

\bibitem[Teng et~al.(2020)Teng, Choromanska, and Campbell]{Teng2020}
Yunfei Teng, Anna Choromanska, and Murray Campbell.
\newblock {Continual learning with direction-constrained optimization}.
\newblock 2020.
\newblock URL \url{http://arxiv.org/abs/2011.12581}.

\bibitem[Tseran et~al.(2018)Tseran, Khan, Harada, and Bui]{Tseran2018b}
Hanna Tseran, Mohammad~Emtiyaz Khan, Tatsuya Harada, and Thang Bui.
\newblock {Natural Variational Continual Learning}.
\newblock \emph{Continual Learning Workshop at the 32nd Conference on Neural
  Information Processing Systems (NeurIPS)}, \penalty0 (NeurIPS):\penalty0
  1--5, 2018.

\bibitem[von Oswald et~al.(2019)von Oswald, Henning, Sacramento, and
  Grewe]{VonOswald2019}
Johannes von Oswald, Christian Henning, Jo{\~{a}}o Sacramento, and Benjamin~F.
  Grewe.
\newblock {Continual learning with hypernetworks}.
\newblock pp.\  1--28, 2019.
\newblock URL \url{http://arxiv.org/abs/1906.00695}.

\bibitem[Yin et~al.(2020)Yin, Farajtabar, and Li]{Yin2020}
Dong Yin, Mehrdad Farajtabar, and Ang Li.
\newblock {SOLA: Continual learning with second-order loss approximation}.
\newblock \emph{arXiv}, pp.\  1--18, 2020.

\bibitem[Yoon et~al.(2018)Yoon, Yang, Lee, and Hwang]{Yoon2018}
Jaehong Yoon, Eunho Yang, Jeongtae Lee, and Sung~Ju Hwang.
\newblock {Lifelong learning with dynamically expandable networks}.
\newblock In \emph{6th International Conference on Learning Representations,
  ICLR 2018 - Conference Track Proceedings}, pp.\  1--11, 2018.

\bibitem[Yoon et~al.(2019)Yoon, Kim, Yang, and Hwang]{Yoon2019}
Jaehong Yoon, Saehoon Kim, Eunho Yang, and Sung~Ju Hwang.
\newblock {Scalable and Order-robust Continual Learning with Additive Parameter
  Decomposition}.
\newblock pp.\  1--15, 2019.
\newblock URL \url{http://arxiv.org/abs/1902.09432}.

\bibitem[Zeng et~al.(2019)Zeng, Chen, Cui, and Yu]{Zeng2019a}
Guanxiong Zeng, Yang Chen, Bo~Cui, and Shan Yu.
\newblock {Continual learning of context-dependent processing in neural
  networks}.
\newblock \emph{Nature Machine Intelligence}, 1\penalty0 (8):\penalty0
  364--372, 2019.
\newblock \doi{10.1038/s42256-019-0080-x}.

\bibitem[Zenke et~al.(2017)Zenke, Poole, and Ganguli]{Zenke2017}
Friedemann Zenke, Ben Poole, and Surya Ganguli.
\newblock {Continual learning through synaptic intelligence}.
\newblock \emph{34th International Conference on Machine Learning, ICML 2017},
  8:\penalty0 6072--6082, 2017.

\end{thebibliography}
